\documentclass[lettersize,journal]{IEEEtran}
\usepackage{amsmath,amsfonts}
\usepackage{amssymb}
\usepackage{algorithmic}
\usepackage{algorithm}
\usepackage{array}
\usepackage[caption=false,font=normalsize,labelfont=sf,textfont=sf]{subfig}
\usepackage{textcomp}
\usepackage{stfloats}
\usepackage{url}
\usepackage{verbatim}
\usepackage{graphicx}
\usepackage{cite}

\usepackage[table]{xcolor}
\def\BibTeX{{\rm B\kern-.05em{\sc i\kern-.025em b}\kern-.08em
    T\kern-.1667em\lower.7ex\hbox{E}\kern-.125emX}}

\usepackage{enumitem}
\usepackage{booktabs}
\usepackage{multirow}
\usepackage{makecell}
\definecolor{mygray}{gray}{.9}
\definecolor{mygreen}{RGB}{206,246,243}
\definecolor{myorange}{RGB}{250,214,190}
\usepackage{caption}
\usepackage{float}
\usepackage{bbding}
\usepackage{threeparttable}

\usepackage{amsthm}
\newtheorem{theorem}{Theorem}
\newtheorem{definition}{Definition}
\newtheorem{proposition}{Proposition}
\newtheorem{property}{Property}

\begin{document}

\title{Exploring the Paradigm Shift from Grounding to Skolemization for Complex Query Answering on Knowledge Graphs}


\author{Yuyin Lu, Hegang Chen, Shanrui Xie, Yanghui Rao*,~\IEEEmembership{Member,~IEEE}, Haoran Xie,~\IEEEmembership{Senior Member,~IEEE},\\ Fu Lee Wang,~\IEEEmembership{Senior Member,~IEEE}, Qing Li,~\IEEEmembership{Fellow,~IEEE}
\thanks{Yuyin Lu, Hegang Chen, Shanrui Xie, and Yanghui Rao are with the School of Computer Science and Engineering, Sun Yat-sen University, Guangzhou, China. Haoran Xie is with the School of Data Science, Lingnan University, Tuen Mun, New Territories, Hong Kong SAR. Fu Lee Wang is with the School of Science and Technology, Hong Kong Metropolitan University, Ho Man Tin, Kowloon, Hong Kong SAR. Qing Li is with the Department of Computing, The Hong Kong Polytechnic University, Hung Hom, Kowloon, Hong Kong SAR. * \textit{Corresponding author. E-mail: raoyangh@mail.sysu.edu.cn.}}
}

% \markboth{Journal of IEEE TRANSACTIONS ON PATTERN ANALYSIS AND MACHINE INTELLIGENCE}%


\maketitle

\begin{abstract}
Complex Query Answering (CQA) over incomplete Knowledge Graphs (KGs), typically formalized as reasoning with Existential First-Order predicate logic with one free variable (EFO\textsubscript{1}), faces a fundamental tradeoff between logic fidelity and computational efficiency. This work establishes a Grounding-Skolemization dichotomy to systematically analyze this challenge and motivate a paradigm shift in CQA. While Grounding-based methods inherently suffer from combinatorial explosion, most Skolemization-based methods neglect to explicitly model Skolem functions and compromise logical consistency. To address these limitations, we propose the Logic-constrained Vector Symbolic Architecture (LVSA), a neuro-symbolic framework that unifies a differentiable Skolemization module and a neural negator, as well as a logical constraint-driven optimization protocol to harmonize geometric and logical requirements. Theoretically, LVSA guarantees universality for all EFO\textsubscript{1} queries with low computational complexity. Empirically, it outperforms state-of-the-art Skolemization-based methods and reduces inference costs by orders of magnitude compared to Grounding-based baselines.
\end{abstract}


\section{Introduction}
\IEEEPARstart{K}{nowledge} Graphs (KGs) are graph-structured databases that represent facts as relationships between entities, providing critical support for information retrieval in various applications, such as recommendation systems~\cite{DBLP:conf/icde/Wu0SGKO23} and drug discovery~\cite{zhang2025comprehensive}. In practice, queries issued by users and intelligent agents range from simple one-hop lookups to complex queries formalized in Existential First-Order predicate logic with one free variable (EFO\textsubscript{1}), which incorporate logical operations including conjunction ($\wedge$), disjunction ($\vee$), negation ($\neg$), and existential quantifier ($\exists$), as exemplified in Fig. \ref{fig:kg-example}(i).

The inherent incompleteness of real-world KGs, combined with the structural complexity of EFO\textsubscript{1} queries, poses fundamental challenges to the accuracy and efficiency of Complex Query Answering (CQA). These challenges establish CQA as a crucial problem in knowledge representation and reasoning~\cite{DBLP:journals/pami/LiangMLLTWZLSH24}. Prior work conventionally categorizes CQA methods into symbolic, neural, and neuro-symbolic paradigms~\cite{zhang2021neural}, emphasizing that traditional symbolic approaches~\cite{DBLP:conf/nips/BordesUGWY13,DBLP:conf/iclr/SunDNT19} often fail on incomplete KGs due to their reliance on deterministic search. In contrast, neural~\cite{zhang2024conditional,DBLP:conf/nips/RenL20} and neuro-symbolic methods~\cite{DBLP:conf/iclr/ArakelyanDMC21,DBLP:conf/icml/Zhu0Z022} achieve superior performance through enhanced generalization capabilities. To enable a more systematic analysis, we introduce a formal logic-grounded taxonomy that vertically extends this conventional neuro-symbolic categorization. By identifying existentially quantified variable handling as the critical determinant of CQA efficacy, we partition existing CQA methods into \textbf{Grounding-based} and \textbf{Skolemization-based} paradigms. This \textbf{Grounding-Skolemization dichotomy} elucidates tradeoffs between computational efficiency and logic fidelity, providing a unified analytical lens to diagnose limitations in current CQA methods.

\begin{figure}[t]
\centering
\includegraphics[width=1.\linewidth]{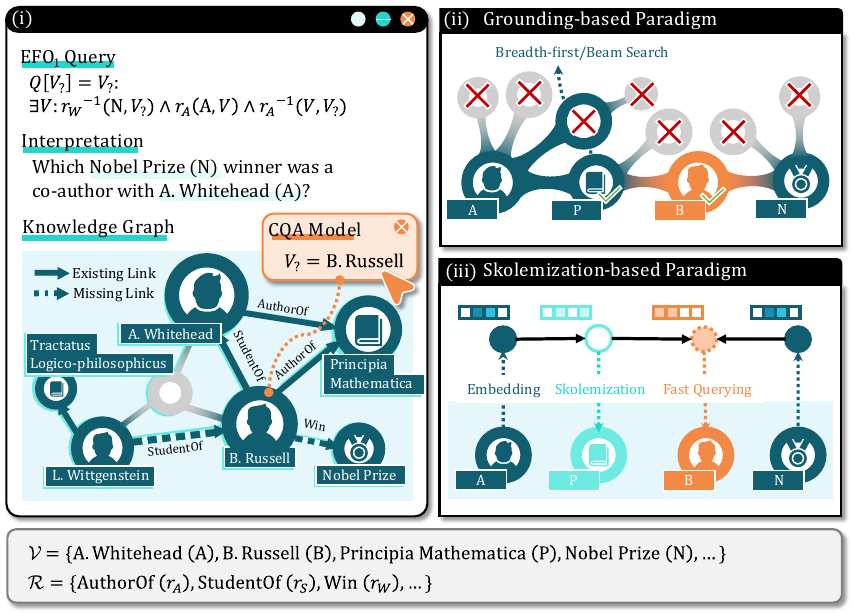}
\caption{\textbf{(i)} An example KG and a complex query in EFO\textsubscript{1} form. $-1$ denotes the inverse relation. \textbf{(ii)} Grounding-based reasoning workflow. \textbf{(iii)} Skolemization-based reasoning workflow.}
\label{fig:kg-example}
\vspace{-15pt}
\end{figure}

The Grounding-based paradigm operates by grounding EFO\textsubscript{1} formulas into propositional logic through candidate enumeration, as illustrated in Fig. \ref{fig:kg-example}(ii). Representative Grounding-based methods~\cite{DBLP:conf/iclr/ArakelyanDMC21,DBLP:conf/icml/Zhu0Z022,DBLP:conf/icml/BaiLLH23,DBLP:conf/iclr/Yin0S24} employ fuzzy logic~\cite{novak2012mathematical} to compositionally aggregate atomic formula truth values inferred by neural link predictors~\cite{trouillon2016complex,DBLP:conf/nips/ZhuZXT21,DBLP:conf/icde/Liang23}. While these neuro-symbolic methods demonstrate strong interpretability, their existentially quantified complexity grows exponentially with the number of existentially quantified variables. Although engineering optimizations~\cite{DBLP:conf/icde/YangHWY16,DBLP:conf/icde/AbuodaTA22} can reduce response time, they cannot overcome the fundamental scalability limitations inherent in this exponential complexity when applied to real-world KGs with complex query structures.

Conversely, the Skolemization-based paradigm replaces existentially quantified variables with neural Skolem functions (Fig. \ref{fig:kg-example}(iii)), achieving superior reasoning efficiency by relaxing the constraint of logical equivalence to satisfiability. Given the anticipated growth of KGs' scales~\cite{hofer2024construction}, this paradigm represents a critical research frontier. In practice, these neural Skolem functions are typically instantiated through geometric embeddings~\cite{DBLP:conf/nips/HamiltonBZJL18,DBLP:conf/nips/RenL20} and neural networks~\cite{DBLP:conf/iclr/WangSWS23,zhang2024conditional}. However, after the initial attempt in LogicE~\cite{luus2021logic}, subsequent works neglect the formal analysis of their logical underpinnings, particularly the failure to enforce the inherent constraints of Skolem functions and logical operations. This oversight reduces them to neural approximations that lack guaranteed logical coherence.

This work first conducts a formal analysis of existing CQA methods through the proposed Grounding-Skolemization dichotomy, examining their tradeoffs between logic fidelity and efficiency. Building on these insights, we propose \textbf{L}ogic-constrained \textbf{V}ector \textbf{S}ymbolic \textbf{A}rchitecture (\textbf{LVSA})\footnote{The source code of our LVSA is available at \url{https://anonymous.4open.science/r/LVSA-110F}.}, a novel neuro-symbolic approach that synergizes the logical framework based on Skolemization with the operational principles of Vector Symbolic Architectures (VSA)~\cite{schlegel2022comparison}, achieving the dual advantages of logic fidelity and computational efficiency. Specifically, LVSA employs sequential reasoning to ensure that all logical dependencies are resolved before inferring the embedding of each variable. By leveraging VSA's efficient algebraic operations over distributed vector representations, LVSA models relational projections and fundamental logical operations, enabling fast inference while maintaining transparency. Based on VSA's modular framework, we introduce a differentiable Skolemization module that explicitly enforces logical dependencies and a neural negator integrated with logic-driven regularization terms to address the challenge of modeling Skolem functions and negation in vector spaces.

Theoretically, we establish that our approach guarantees universal applicability to arbitrary EFO\textsubscript{1} queries while maintaining a proven computational complexity advantage. Empirically, comprehensive evaluations demonstrate significant improvements over state-of-the-art Skolemization-based methods in reasoning capability. Furthermore, efficiency analysis validates the superior scalability of the Skolemization-based paradigm compared to Grounding-based methods, with LVSA exhibiting additional efficiency gains due to its lightweight architecture. Our approach also shows stronger generalization capability in more challenging CQA scenarios with limited prior knowledge. Finally, through quantitative evaluation of approximated Skolem functions, we confirm that LVSA ensures high logic fidelity, reflected in its enhanced interpretability and causal traceability. Collectively, our theoretical analysis and the empirical success of LVSA substantiate the compelling promise of the Skolemization-based paradigm, paving the way for future research in this undervalued direction.

\section{Preliminary}
\label{sec:pre}
\subsection{Problem Formalization}
\noindent A {\bfseries Knowledge Graph (KG)} can be denoted as $\mathcal{G}=\left\{ \mathcal{V},\mathcal{R},\mathcal{E} \right\}$, where $\mathcal{V}$ and $\mathcal{R}$ are the sets of all entities and relations, respectively, and $\mathcal{E}\subseteq\mathcal{V}\times\mathcal{R}\times\mathcal{V}$ is a set of true assertions in the form of triples. Under the open-world assumption~\cite{russell2016artificial}, an assertion triple of $(e_h,r,e_t)\notin\mathcal{E}$ with $e_h,e_t\in\mathcal{V}$ and $r\in\mathcal{R}$ may also be true.

\noindent {\bfseries Complex Queries on KGs} can be formalized as \textbf{Existential First-Order predicate formulas with one free variable (EFO\textsubscript{1})}.  Formally, a valid EFO\textsubscript{1} query $Q$ consisting of an anchor entity set $\mathcal{V}_h\subseteq \mathcal{V}$, a set of existentially quantified variables $\mathcal{V}_\exists=\{V_1,V_2,\dots , V_k \}$, and a free variable $V_{?}$, can be defined in the Disjunctive Normal Form (DNF) as follows:
\begin{equation}
\label{eq:epo1-formula}
    Q[V_?]=V_?:\exists V_1,\dots,V_k:c_1 \vee c_2 \vee \dots \vee c_n,
\end{equation}
where each $c$ is a conjunctive query, that is, $c_i=q_{i1} \wedge q_{i2} \wedge \dots \wedge q_{im}$. And each literal  $q$ is an atomic formula or its negation, that is, $q_{ij}=r(V',V)$ or $q_{ij}=\neg r(V',V)$, with head $V'\in\mathcal{V}_h \cup \mathcal{V}_\exists$, tail $V\in \{V_?\}\cup \mathcal{V}_\exists$, and relation $r\in\mathcal{R}$. The answer set of the complex query $Q[V_?]$ is denoted as $A[Q]\subseteq\mathcal{V}$.

\subsection{First-order Predicate Logic Reasoning}
A fundamental challenge in first-order predicate logic reasoning lies in handling existentially quantified variables (hereafter, \textit{existential variables}). To address this problem, \textbf{Grounding} and \textbf{Skolemization} are two primary strategies, respectively based on Herbrand's Theorem and Skolemn's Theorem~\cite{van1967frege}.

\begin{definition}[Grounding]
Given an EFO\textsubscript{1} formula $\exists V: r(V', V)$, grounding converts it to a ground formula by enumerating all candidate entities for substitutions $\{e/V| e \in \mathcal{V}\}$:
\begin{equation}
    \exists V: r(V', V) \Rightarrow_G \bigvee_{e\in \mathcal{V}} r(V', e).
\end{equation}
\end{definition}

\begin{property}[Properties of Grounding]
Let $F$ be an EFO\textsubscript{1} formula and $F \Rightarrow_G F_G$: 
\begin{enumerate}[label=\alph*.]
    \item \textbf{Equivalence}: $F$ and $F_G$ are logically equivalent.
    \item \textbf{Complexity}: Let $F$ involve $k$ existential variables.
    \begin{enumerate}[label=(\roman*)]
        \item \textbf{Best-case}: When existential variables are independent, grounding performs $k$ separate substitutions, yielding linear complexity $\mathcal{O}(k \cdot |\mathcal{V}|)$.
        \item \textbf{Worst-case}: When existential variables form a chain of dependencies, grounding requires chained substitutions, yielding exponential complexity $\mathcal{O}(|\mathcal{V}|^{k})$.
\end{enumerate}
\end{enumerate}
\label{property:grounding}
\end{property}

\begin{definition}[Skolemization]
Given EFO\textsubscript{1} formulas $F_D=\exists V_D: r(V', V_D)$ and $F_I=\exists V_I: r(V_I, V)$ with dependent existential variable $V_D$ and independent existential variable $V_I$, Skolemization eliminates the existential quantifier by respectively introducing a Skolem function $f_r(V')$ and a Skolem constant $f_r()$, as follows:
\begin{equation}
    F_D \Rightarrow_S r\left(V',f_r(V')\right),~F_I \Rightarrow_S r\left(f_r(),V\right).
\end{equation}
\end{definition}

\begin{property}[Properties of Skolemization]
Let $F$ be an EFO\textsubscript{1} formula and $F \Rightarrow_S F_S$: 
\begin{enumerate}[label=\alph*.]
    \item \textbf{Satisfiability}: $F$ is satisfiable $\Leftrightarrow$ $F_S$ is satisfiable. $F_S \models F$ but the converse ($F \models F_S$) is not true in general.
    % \item \textbf{Satisfiability}: $F$ is satisfiable $\Leftrightarrow$ $F_S$ is satisfiable, but $F \not\equiv F_S$.
    \item \textbf{Complexity}: When $F$ involves $k$ existential variables, the complexity of Skolemization is $\mathcal{O}(k)$.
\end{enumerate}
\label{property:skolem}
\end{property}

\noindent \textbf{Example 1} \textit{Consider the complex query} $Q[V_?]=V_?: \exists V: r_A(\text{A}, V) \wedge r_A^{-1}(V, V_?) \wedge r_W^{-1}(\text{N}, V_?)$ \textit{illustrated in Fig. \ref{fig:kg-example}. Grounding-based reasoning is defined by}
\begin{equation}
\begin{aligned}
    & Q[V_?] \Rightarrow_G Q_G[V_?]= V_?: \\
    & \left( r_A(\text{A}, \text{P}/V) \wedge r_A^{-1}(\text{P}/V, V_?) \wedge r_W^{-1}(\text{N}, V_?) \right) \vee  \\
    & \left( r_A(\text{A}, \text{B}/V) \wedge r_A^{-1}(\text{B}/V, V_?) \wedge r_W^{-1}(\text{N}, V_?) \right) \vee \cdots
\end{aligned}
\end{equation}
\textit{In contrast, Skolemization-based reasoning is defined by}
\begin{equation}
\begin{aligned}
    & Q[V_?] \Rightarrow_S Q_S[V_?] = V_?: \\
    & r_A(\text{A}, f_{r_A}(\text{A})) \wedge r_A^{-1}(f_{r_A}(\text{A}), V_?) \wedge r_W^{-1}(\text{N}, V_?).
\end{aligned}
\end{equation}

\section{The Grounding-Skolemization Dichotomy}
\label{sec:dichotomy}
This section presents a systematic analysis of current CQA methods through the lens of the Grounding-Skolemization dichotomy. By examining their algorithmic ideas and technical designs, we identify prevailing challenges and  propose a promising approach to guide future research in the field.

\subsection{Algorithmic Ideas}
\noindent \textbf{Grounding-based Methods}~\cite{DBLP:conf/iclr/ArakelyanDMC21,DBLP:conf/icml/Zhu0Z022,DBLP:conf/icml/BaiLLH23,DBLP:conf/iclr/Yin0S24,DBLP:conf/www/Nguyen0SH025,DBLP:conf/nips/GalkinZ00Z24} ground EFO\textsubscript{1} queries into propositional logic formulas and reduce the truth value of each formula to an aggregation of its constituent literals. The core modules of most Grounding-based methods comprise a neural link predictor $\phi: \mathcal{V}\times\mathcal{R}\times\mathcal{V} \mapsto [0, 1]$ that estimates the truth values of literals and a set of fuzzy logic operators $\{ \top: \wedge, \bot: \vee, \text{Neg}: \neg\}$ for truth value composition, where $\top$ and $\bot$ are specific t-norm/t-conorm pairs~\cite{novak2012mathematical}, and $\text{Neg}$ is a negator. 

Given the equivalence of Grounding (Property \ref{property:grounding}a) and the truth functionality of propositional logic~\cite{cook2009dictionary}, the grounding-reduction process preserves logical equivalence. However, as shown in Property \ref{property:grounding}b, most Grounding-based methods~\cite{DBLP:conf/www/Nguyen0SH025,DBLP:conf/icml/Zhu0Z022,DBLP:conf/icml/BaiLLH23,DBLP:conf/iclr/Yin0S24} suffer from exponential complexity due to combinatorial explosion, rendering them impractical for queries with $k > 3$ in various KGs~\cite{DBLP:conf/kdd/RenDDCZLS22}. To alleviate this problem, CQD-Beam~\cite{DBLP:conf/iclr/ArakelyanDMC21} employs beam search to prune the size of the search space from $|\mathcal{V}|^k$ to $b^k \cdot |\mathcal{V}|$ with the beam size $b$. Yet, this greedy strategy trades completeness and accuracy for tractability, which is quantified in Section \ref{sec:efficiency-exp}.

\noindent \textbf{Skolemization-based Methods}~\cite{DBLP:conf/nips/HamiltonBZJL18,DBLP:conf/naacl/BaiWZS22,DBLP:conf/iclr/WangSWS23,zhang2024conditional,DBLP:conf/nips/RenL20,DBLP:conf/nips/ZhangWCJW21,DBLP:conf/ijcai/ZhuoPWW0LW025} project entities into a particular embedding space and then define logical operations as transformations within this space. Formally, we denote the embedding function as $\varphi$ and transformation functions as $\left\{g_P, g_\wedge, g_\vee, g_\neg, g_\exists \right\}$, where $g_P, g_\wedge, g_\vee, g_\neg$ are respectively transformation functions for relational projection, conjunction, disjunction, negation, and $g_\exists$ is the neural Skolem function. 
By recursively applying these transformations according to a query's logical structure, these methods derive the query embedding $\varphi(Q)$, thereby facilitating fast querying based on the similarity metric ($\phi_E$) defined in the embedding space, where the score of a candidate answer $e\in\mathcal{V}$ is quantified by $E\left(Q_S[e]\right)=\phi_E\left(\varphi(Q),\varphi(e)\right)$.

As formalized in Property \ref{property:skolem}, Skolemization enhances reasoning efficiency by relaxing logical equivalence to satisfiability, guaranteeing the discovery of at least one valid witness. However, current Skolemization-based methods often underperform their Grounding-based counterparts on standard benchmarks~\cite{DBLP:conf/nips/RenL20}, due to this logical relaxation and the difficulty in approximating Skolem functions. Furthermore, existing Skolemization-based methods typically rely on complex neural components (e.g., geometric embeddings~\cite{DBLP:conf/nips/ZhangWCJW21} and deep networks~\cite{zhang2024conditional}), which compromise reasoning transparency. Despite these limitations, new benchmarks highlight the critical importance of CQA in knowledge-sparse scenarios~\cite{greguccicomplex}. Our experimental analysis in Section \ref{sec:new-benchmark-results} confirms that Skolemization-based methods possess a superior generalization potential in such settings. Given their inherent efficiency and generalization capability, we contend that the Skolemization-based paradigm constitutes a vital and promising research direction.

\noindent \textbf{Example 1 (cont.)} \textit{Considering }$V_?=\text{B}$\textit{, the reasoning algorithm of Grounding-based methods is formalized as follows:}
\begin{equation}
\begin{aligned}
    & E\left(Q_G[\text{B}]\right) = \\
    & \left( \phi(\text{A},r_A,\text{P})~\top~ \phi(\text{P},r_A^{-1},\text{B})~\top~\phi(\text{N},r_W^{-1},\text{B}) \right) \bot \\
    & \left( \phi(\text{A},r_A,\text{B})~\top~ \phi(\text{B},r_A^{-1},\text{B})~\top~\phi(\text{N},r_W^{-1},\text{B}) \right) \bot \cdots
\end{aligned}
\end{equation}

\textit{The reasoning algorithm of Skolemization-based methods is defined by}
\begin{equation}
\begin{aligned}
    & f_{r_A}(\text{A}) \approx g_\exists(\text{A}~|~Q), \\
    &\varphi(Q) = g_\wedge \left( g_P\left( f_{r_A}(\text{A})~|~r_A^{-1} \right), g_P\left( \text{N}~|~r_W^{-1} \right) \right), \\ 
    & E\left(Q_S[\text{B}]\right) = \phi_E\left(\varphi(Q), \varphi(\text{B})\right).
\end{aligned}
\end{equation}

\begin{figure}[t]
  \centering
  \includegraphics[width=1\linewidth]{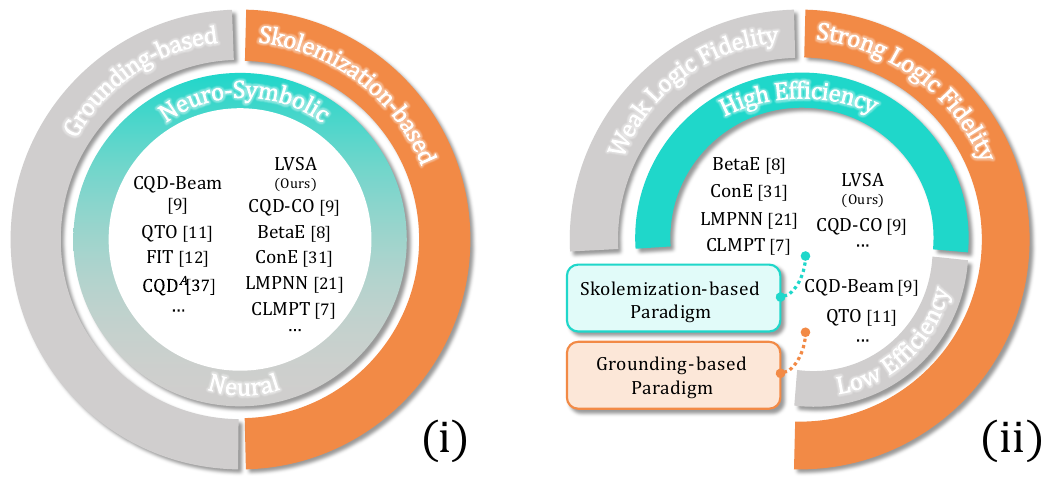}
  \caption{\textbf{(i)} Taxonomy of CQA methods along the Grounding-Skolemization and Neuro-Symbolic dimensions. \textbf{(ii)} Comparison of Skolemization-based and Grounding-based CQA methods in efficiency and logic fidelity.}
  \label{fig:CQA-taxonomy}
\end{figure}

\begin{table*}[t]
  \caption{Summary of representative Skolemization-based methods. \textbf{Bold} entries denote symbolic operations. $\varphi(e)$ and $\varphi(r)$ denote the embedding of an entity and a relation, respectively, and $d$ is the embedding dimension.}
  \label{tab:architectures}
  \centering
  \resizebox{1\linewidth}{!}{%
  \begin{tabular}{c|c|c|c|c|c|c|c}
    \toprule[1.5pt]
    \multirow{2}{*}{\textbf{Model}}  & \textbf{Embedding Space} & \textbf{Relation} & \textbf{Existential} & \multirow{2}{*}{\textbf{Conjunction}} & \multirow{2}{*}{\textbf{Disjunction}} & \multirow{2}{*}{\textbf{Negation}} & \textbf{Similarity} \\
    & Entity: $\varphi(e)$~Relation: $\varphi(r)$ & \textbf{Projection} & \textbf{Quantification} & & & & \textbf{/Distance} \\
    \midrule
    \midrule
    \multirow{2}{*}{Q2B~\cite{DBLP:conf/iclr/RenHL20}} & \textit{vector}: $\varphi(e)\in\mathbb{R}^{d}$ & \multirow{2}{*}{\textbf{addition}} & \multirow{2}{*}{box} & normalized addition & \multirow{2}{*}{\textbf{DNF}} & \multirow{2}{*}{\XSolidBrush} & entity-to-box \\
    & \textit{box}: $\varphi(r)=(\text{Cen}(r),\text{Off}(r))\in\mathbb{R}^{2d}$ & & & + neural network & & & distance \\
    \midrule
    \multirow{2}{*}{BetaE~\cite{DBLP:conf/nips/RenL20}} & \textit{beta distribution}: & neural & beta & normalized addition & \multirow{2}{*}{\textbf{DNF}} & \multirow{2}{*}{\textbf{reciprocal}} & KL \\
    & $\varphi(e)=\text{Beta}(\alpha_e,\beta_e)\in\mathbb{R}^{2d}$ & network & distribution & + neural network & & & divergency \\
    \midrule
    \multirow{2}{*}{ConE~\cite{DBLP:conf/nips/ZhangWCJW21}} & \textit{cone}: $\varphi(e)=(\theta_{\text{ax},e},\theta_{\text{ap},e})\in[0,2\pi]^{2d}$ & addition + & \multirow{2}{*}{cone} & normalized addition & \multirow{2}{*}{\textbf{DNF}} & \multirow{2}{*}{\textbf{complement}} & cone \\
    & \textit{cone}: $\varphi(r)=(\theta_{\text{ax},r},\theta_{\text{ap},r})\in[0,2\pi]^{2d}$ & neural network & & + neural network & & & distance
    \\
    \midrule
    \multirow{2}{*}{FuzzQE~\cite{DBLP:conf/aaai/ChenHS22}} & \multirow{2}{*}{\textit{fuzzy set}: $\varphi(e)\in[0,1]^{d}$} & \multirow{2}{*}{neural network} & \multirow{2}{*}{fuzzy set} & \multirow{2}{*}{\textbf{t-norm}} & \multirow{2}{*}{\textbf{t-conorm}} & \multirow{2}{*}{\textbf{complement}} & cosine \\
    & & & & & & & similarity  \\
    \midrule
    \multirow{2}{*}{LogicE~\cite{luus2021logic}} & \textit{fuzzy set}: $\varphi(e)=(l_e,u_e)\in[0,1]^{2d}$ & \multirow{2}{*}{neural network} & \multirow{2}{*}{fuzzy set} & \multirow{2}{*}{\textbf{t-norm}} & \multirow{2}{*}{\textbf{t-conorm}} & \multirow{2}{*}{\textbf{complement}} & set \\
    & \textit{vector}: $\varphi(r)\in\mathbb{R}^{d}$ & & & & & & distance \\
    \midrule
    \midrule
    \multirow{2}{*}{CQD-CO~\cite{DBLP:conf/iclr/ArakelyanDMC21}} & \textit{complex vector}: $\varphi(e)\in\mathbb{C}^{2d}$ & \textbf{Hadamard} & complex & \multirow{2}{*}{\textbf{t-norm}} & \multirow{2}{*}{\textbf{t-conorm}} & \multirow{2}{*}{\textbf{complement}} & cosine \\
    & \textit{complex vector}: $\varphi(r)\in\mathbb{C}^{2d}$ & \textbf{product} & vector & & & & similarity \\
    \midrule
    \multirow{2}{*}{PPMT~\cite{DBLP:conf/ijcai/ZhuoPWW0LW025}} & \textit{vector}: $\varphi(e)\in\mathbb{R}^{d}$ & \multirow{2}{*}{neural network} & \multirow{2}{*}{neural network} & \multirow{2}{*}{neural network} & \multirow{2}{*}{\textbf{DNF}} & \multirow{2}{*}{\XSolidBrush} & cosine \\
    & \textit{vector}: $\varphi(r)\in\mathbb{R}^{d}$ & & & & &  & similarity \\
    \midrule
    LMPNN~\cite{DBLP:conf/iclr/WangSWS23} & \textit{complex vector}: $\varphi(e)\in\mathbb{C}^{2d}$ & \multirow{2}{*}{neural network} & \multirow{2}{*}{neural network} & \multirow{2}{*}{neural network} & \multirow{2}{*}{\textbf{DNF}} & neural & cosine \\
    CLMPT~\cite{zhang2024conditional} & \textit{complex vector}: $\varphi(r)\in\mathbb{C}^{2d}$ & & & & & network & similarity \\
    \midrule
    \midrule
    \multirow{2}{*}{LVSA (Ours)} & \textit{complex vector}: $\varphi(e)\in\mathbb{C}^{2d}$ & \textbf{Hadamard} & \multirow{2}{*}{neural network} & \multirow{2}{*}{\textbf{normalized addition}} & \multirow{2}{*}{\textbf{DNF}} & neural & cosine \\
    & \textit{complex vector}: $\varphi(r)\in\mathbb{C}^{2d}$ & \textbf{product} &  & & & network & similarity \\
    \bottomrule[1.5pt]
  \end{tabular}
  }
\end{table*}

\subsection{Technical Designs}
\noindent\textbf{Grounding-based Methods:} Pioneered by Arakelyan et al.~\cite{DBLP:conf/iclr/ArakelyanDMC21}, the Grounding-based paradigm operates through atomic formula evaluation via neural link predictors~\cite{DBLP:conf/nips/ZhuZXT21,trouillon2016complex,DBLP:conf/icml/NickelTK11} and truth value aggregation using fuzzy logic operators (e.g., G\"odel t-norms~\cite{novak2012mathematical}). Subsequent advancements include GNN-QE~\cite{DBLP:conf/icml/Zhu0Z022} and ULTRAQ~\cite{DBLP:conf/nips/GalkinZ00Z24}, which integrate GNN-based link predictors, and CQD$^\mathcal{A}$~\cite{DBLP:conf/nips/ArakelyanMDCA23}, which enhances reasoning through neural adapter modules. Concurrently, SpherE~\cite{DBLP:conf/www/Nguyen0SH025} investigates the efficacy of diverse geometric embedding-based link predictors within the Grounding-based CQA framework. In parallel, Lu et al.~\cite{luefficient} replace classical fuzzy operators with a fuzzy system, improving the capability of uncertainty handling through fuzzy rules. From the perspective of grounding algorithms, QTO~\cite{DBLP:conf/icml/BaiLLH23} and FIT~\cite{DBLP:conf/iclr/Yin0S24} guarantee global optimality through exhaustive candidate grounding, albeit at the cost of exponential time complexity. Instead, CQD-Beam~\cite{DBLP:conf/iclr/ArakelyanDMC21} employs beam search to prune the candidate set, trading reasoning performance and logical completeness for tractability.

\noindent\textbf{Skolemization-based Methods:}
Although most CQA methods based on query embedding learning lack formal analysis of their logical foundations, their handling of existential variables can be conceptually aligned with neural approximations of Skolem functions. Current approaches are broadly categorized into geometric embedding-based and neural architecture-based methods. Table \ref{tab:architectures} compares the model architectures of representative Skolemization-based methods.

Geometric embedding-based methods for CQA design geometrically structured embedding spaces (e.g., cones~\cite{DBLP:conf/nips/ZhangWCJW21}, boxes~\cite{DBLP:conf/iclr/RenHL20}, and beta distributions~\cite{DBLP:conf/nips/RenL20}) based on set-theoretic principles~\cite{vilnis2021geometric}, which encode logical inductive biases through geometric priors. Specifically, these methods represent existential variables as geometric embeddings of satisfiable solution sets and implement logical operators through algebraic transformations defined over the geometric space. For instance, ConE~\cite{DBLP:conf/nips/ZhangWCJW21} employs cone-sector embeddings and models logical conjunction and negation via cone intersections and complements, respectively. FuzzQE~\cite{DBLP:conf/aaai/ChenHS22} and LogicE~\cite{luus2021logic} map (non-)variables to fuzzy set embeddings and model logical operations through fuzzy set operators in the embedding space. However, geometric embedding-based methods rely on intricate symbolic-geometric operations, leading to optimization instability and high computational costs. Moreover, they suffer from performance degradation due to misalignment between geometric priors and actual logical constraints.

In contrast, neural architecture-based methods adopt simple vector embeddings and approximate Skolem functions through neural networks. CQD-CO~\cite{DBLP:conf/iclr/ArakelyanDMC21} suffers from oversimplified Skolem approximations through random initialization of existential variable embeddings with na\"ive constraint optimization. Subsequently, LMPNN~\cite{DBLP:conf/iclr/WangSWS23} and CLMPT~\cite{zhang2024conditional} leverage Graph Isomorphism Networks (GINs)~\cite{DBLP:conf/iclr/XuHLJ19} to generate embeddings for existential variables via message passing over query graphs. Another direction explored by PPMT~\cite{DBLP:conf/ijcai/ZhuoPWW0LW025} serializes the query graphs and employs Language Models (LMs) to derive query embeddings. However, these LM-based methods show no clear advantage over GNN-based alternatives and lack support for logical negation, revealing notable limitations. Although GNN-based and LM-based methods achieve performance gains via deep neural networks, their black-box reasoning mechanisms fundamentally obscure logical transparency and prevent explicit constraint enforcement.

\subsection{Discussions}
\noindent\textbf{Relation to Neuro-Symbolic Taxonomy:}
Although conventional surveys~\cite{zhang2021neural} employ a neuro-symbolic taxonomy to analyze CQA methods, we contend that the boundaries within this taxonomy are often blurred, as illustrated in Fig.~\ref{fig:CQA-taxonomy}(i). For instance, while prior studies typically categorize CQA methods based on geometric embeddings as purely neural approaches, our analysis reveals their underappreciated symbolic dimensions and hybrid nature (Table \ref{tab:architectures}). This limitation underscores the critical value of the proposed Grounding-Skolemization dichotomy for systematically analyzing both current and emerging CQA methods.

\noindent\textbf{Logic Fidelity-Efficiency Tradeoff:}
Our analysis of both algorithmic ideas and technical designs exposes a fundamental tradeoff between logic fidelity and computational efficiency in CQA methods, as illustrated in Fig.~\ref{fig:CQA-taxonomy}(ii). At the paradigm level, Skolemization-based methods achieve greater efficiency when reasoning over existential variables by relaxing logical equivalence, whereas Grounding-based methods favor completeness through exhaustive search. At the method level, even within the Skolemization-based paradigm, different methods exhibit considerable variation. Specifically, we attribute their differences in efficiency primarily to model complexity, while differences in logic fidelity are seen in their symbolic formalism, transparency, and interpretability.

\noindent\textbf{A Promising Solution:}
Overall, we suggest that the current Skolemization-based paradigm is undervalued in terms of its generalization capability and efficiency, and that previous implementations based on geometric embeddings or deep neural networks have not fully realized its potential.
To enhance reasoning efficiency with logic fidelity, we propose the \textbf{L}ogic-constrained \textbf{V}ector \textbf{S}ymbolic \textbf{A}rchitecture \textbf{(LVSA)}, a novel neuro-symbolic solution for the Skolemization-based CQA paradigm. LVSA performs transparent, step-by-step reasoning where the embedding of each variable is explicitly derived from the resolved results of its predecessors, ensuring full causal traceability. To increase symbolic formalism and reasoning transparency, LVSA adopts the Vector Symbolic Architecture (VSA) as its computational foundation, which provides powerful algebraic primitives on simple vector embeddings~\cite{schlegel2022comparison}. The modular design of VSA enables the integration of logic-constrained neural components, including a neural negator and a differentiable Skolemization module. 

\section{Methodology}
\label{sec:method}
The key modules of our LVSA can also be summarized as $\left( \varphi, \left\{g_P, g_\wedge, g_\vee, g_\neg, g_\exists \right\}, \phi_E \right)$. Firstly, we introduce the basic vector symbolic encoding (Section \ref{sec:basic-symbolic}), followed by a neural negator and a differentiable Skolemization module (Section \ref{sec:multi-head-skolem}). Then, we present a logic-driven optimization protocol for preserving logical constraints (Section \ref{sec:logical-constraint-op}).

\subsection{Basic Vector Symbolic Encoding}
\label{sec:basic-symbolic}
\noindent \textbf{Sequential Reasoning Strategy:} Any EFO\textsubscript{1} query can be represented as a Directed Acyclic Graph (DAG), where nodes correspond to anchor entities or variables and edges encode relational projections~\cite{DBLP:conf/iclr/Yin0S24}. While GNN-based methods update all node embeddings simultaneously through message passing over the entire query graph, LVSA executes reasoning sequentially by following a topological ordering of nodes~\cite{donald1999art}. We employ Kahn's algorithm~\cite{DBLP:journals/cacm/Kahn62} for topological sorting, which iteratively selects nodes ($V$) with zero in-degree ($d_V=0$) and removes them from the graph. This strategy ensures that all logical dependencies of a variable are resolved before its embedding is computed, thereby preserving logic fidelity throughout the multi-step inference process.

\noindent \textbf{Example 2} \textit{Consider the complex query illustrated in Fig. \ref{fig:topo-example}. LVSA performs reasoning following the topological order $(V_1,h,V_2,V_?)$: it first infers embeddings of $V_1$ and the anchor entity $h$, then derives the embedding of $V_2$ based on these intermediate results. This sequential approach maintains transparent causal relationships at each step. In contrast, GNN-based methods update the embeddings of $V_1$, $V_2$, and $V_?$ simultaneously, leading to entangled representations without explicit causal dependencies. }

\begin{figure}
\centering
\includegraphics[width=1.\linewidth]{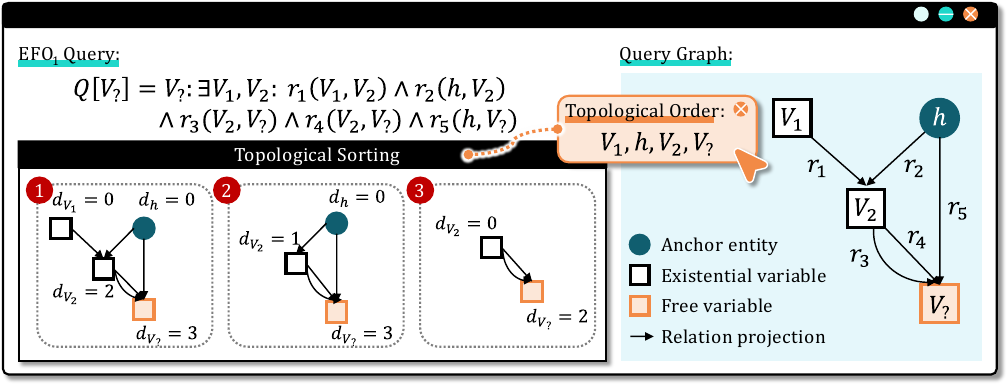}
\caption{Diagram of the query graph and topological sorting workflow for an example query.}
\label{fig:topo-example}
\end{figure}

\begin{figure*}
  \centering
  \includegraphics[width=1.\textwidth]{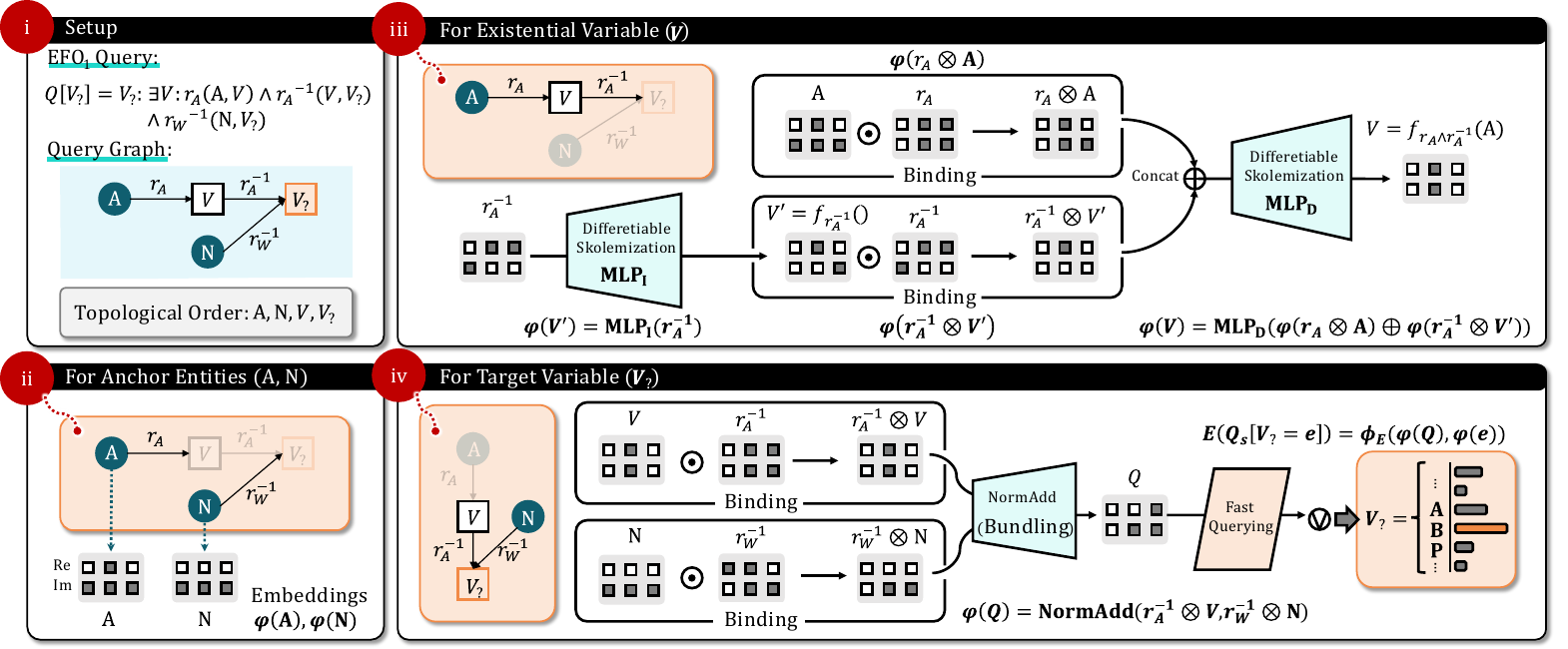}
  \caption{LVSA's inference process for the sample complex query $Q[V_?]=V_?: \exists V: r_A(\text{A}, V) \wedge r_A^{-1}(V, V_?) \wedge r_W^{-1}(\text{N}, V_?)$.}
  \label{fig:LVSA-sample}
\end{figure*}

\noindent \textbf{Embedding Function ($\varphi$):} Building on the success of complex-valued vectors in both Knowledge Graph Embedding (KGE) models~\cite{trouillon2016complex,DBLP:conf/aaai/NickelRP16,DBLP:conf/iclr/SunDNT19} and VSA~\cite{plate1995holographic,plate2003holographic}, we map entities into a complex vector space. For an entity $e\in\mathcal{V}$, its complex-valued embedding is defined by
\begin{equation}
    \varphi(e)\triangleq\text{Re}(e)+\text{Im}(e)i,
\label{eq:complex-emb}
\end{equation}
where the trainable real component $\text{Re}(e) \in \mathbb{R}^{d}$ and imaginary component $\text{Im}(e) \in \mathbb{R}^{d}$  replace traditional VSA's random vectors to capture semantic correlation among entities. Furthermore, we define the concatenated vector representation of $\varphi(e)$ as $v_\varphi(e)\triangleq \left[\text{Re}(e);\text{Im}(e)\right]$.

\noindent \textbf{Relational Projection ($g_P$):} Through projecting relations into the same embedding space as entities, where each relation $r \in \mathcal{R}$ is represented as $\varphi(r) \triangleq \text{Re}(r) + \text{Im}(r)i$, we can define relational projection by the \textbf{Binding ($\otimes$)} operator, which is implemented as Hadamard product ($\odot$) between embeddings of the relation ($r\in\mathcal{R}$) and the head ($V'$):
\begin{equation}
\begin{aligned}
    \varphi(r \otimes V') & \triangleq \text{Re}(r \otimes V')+\text{Im}(r \otimes V')i \\
    & = \left[ \text{Re}(r)\odot\text{Re}(V') - \text{Im}(r)\odot\text{Im}(V') \right] \\
    & + \left[ \text{Re}(r)\odot\text{Im}(V') + \text{Im}(r)\odot\text{Re}(V') \right]i.
\end{aligned}
\end{equation}

Intuitively, binding emphasizes specific attributes of each entity under different relations, e.g., $\text{StudentOf}\otimes\text{A}\approx\text{B}$ and $\text{AuthorOf}\otimes\text{A}\approx\text{P}$ distinguish A. Whitehead's roles as an advisor versus an author.

\noindent \textbf{Logical Conjunction ($g_\wedge$):} Inspired by VSA, we model the conjunction operator through \textbf{Bundling ($+$)}. Notably, bundling inherently preserves geometric proximity between output embeddings (superposition) and input embeddings when it satisfies the commutativity and associativity axioms of logical conjunction. Given a conjunctive query $Q_\wedge = r_1(V'_1,V)\wedge\dots\wedge r_N(V'_N,V)$, its embedding is defined as follows:
\begin{equation}
\label{eq:conjunction}
\begin{aligned}
    \varphi(Q_\wedge) &\triangleq \text{Re}(Q_\wedge)+\text{Im}(Q_\wedge)i \\ 
    &=\varphi(r_1 \otimes V'_1 + \dots + r_N \otimes V'_N).
\end{aligned}
\end{equation}

To enforce magnitude consistency between inputs and the output, bundling is implemented as normalized addition (NormAdd), which is formalized as
\begin{align}
\label{eq:normadd}
&
\begin{aligned}
    \varphi(Q_\wedge)
    &= \text{NormAdd}\left[ \varphi(r_1 \otimes V'_1),\dots, \varphi(r_N \otimes V'_N)\right] \\
    &= \frac{L}{\tilde{L}}\cdot \tilde{\varphi}(Q_\wedge),
\end{aligned} \\
&
\begin{aligned}
    \tilde{\varphi}(Q_\wedge) &\triangleq \tilde{\text{Re}}(Q_\wedge)+\tilde{\text{Im}}(Q_\wedge)i \\
    & = \frac{1}{N}\left[ \sum_{j=1}^{N}\text{Re}(r_j\otimes V_j^{\prime}) + \sum_{j=1}^{N}\text{Im}(r_j\otimes V_j^{\prime})i \right], 
\end{aligned}
\end{align}
where $\tilde{L}=\frac{1}{2d} \Vert v_{\tilde{\varphi}}(Q_\wedge)\Vert_1$ is the average norm of the unnormalized vector $v_{\tilde{\varphi}}(Q_\wedge)$, and $L = \frac{1}{2d \cdot N} \sum_{j=1}^N \Vert v_{\varphi}(r_j\otimes V_{j}^{\prime})\Vert_1$ is the target norm, calculated as the mean norm of input vectors $\{v_{\varphi}(r_j\otimes V_{j}^{\prime})\}_{j=1...N}$.

\noindent \textbf{Logical Disjunction ($g_\vee$):} DNF can decompose the answer set of a disjunctive query into a union of the answer sets of its conjunctive sub-queries, i.e., $A[Q] = A[c_1] \cup \dots \cup A[c_n]$ for $Q[V_?]$.

\noindent \textbf{Similarity Measure ($\phi_E$):} We adopt the real part of the Hermitian inner product as the similarity between complex-valued vectors. Formally, the predicted score of a candidate entity $e\in\mathcal{V}$ for the complex query $Q[V_?]$ is defined by
\begin{equation}
\begin{aligned}
    E(Q_S[e]) &= \text{Re}\left(\langle \varphi(Q), \overline{\varphi(e)} \rangle \right) \\ 
    &= \text{Re}\left( Q \right) \cdot \text{Re}\left( e \right) + \text{Im}\left( Q \right) \cdot \text{Im}\left( e \right),
\end{aligned}
\end{equation}
where $\overline{\varphi(e)}=\text{Re}(e) - \text{Im}(e)i$ is the conjugate of $\varphi(e)$.

\subsection{Neural Negator and Differentiable Skolemization}
\label{sec:multi-head-skolem}
\noindent \textbf{Logical Negation ($g_\neg$):} Modeling semantic inversion in the embedding space while satisfying logical axioms of negation presents fundamental challenges. To address this, we propose a neural negator ($\text{MLP}_\text{N}$) trained under logic-driven regularization (detailed in Section \ref{sec:logical-constraint-op}). As negation is more practical when taking a conjunction together~\cite{DBLP:conf/nips/RenL20}, we define the general form of complex queries involving logical negation as $Q_\neg = Q_\wedge \wedge \neg r_{N+1}(V'_{N+1},V)$. Formally, the embedding of $Q_\neg$ is computed by
\begin{equation}
\begin{aligned}
    & \varphi(\neg(r_{N+1} \otimes V^\prime_{N+1})) = \text{MLP}_\text{N}\left(v_\varphi(Q_\wedge) \oplus v_\varphi(r_{N+1} \otimes V^\prime_{N+1})\right), \\
    & \varphi(Q_\neg) = \text{NormAdd}\left[ \varphi(Q_\wedge), \varphi(\neg(r_{N+1} \otimes V^\prime_{N+1})) \right],
\end{aligned}
\label{eq:negation}
\end{equation}
where $\text{MLP}_\text{N}(\cdot)$ is a feedforward neural network and $\oplus$ denotes embedding concatenation.

\noindent \textbf{Existential Quantification ($g_\exists$):} A key contribution of our LVSA lies in explicitly modeling Skolem functions under first-order logical dependencies, which inherently strengthens logic fidelity while theoretically guaranteeing universal expressivity for EFO\textsubscript{1} queries.

For an independent existential variable $V_I$, as in $\exists V_I: r(V_I, V) \Rightarrow_S r\left( f_r(), V \right)$, we generate its embedding via a relation-aware neural network ($\text{MLP}_\text{I}$):
\begin{equation}
    \varphi(V_I) = \text{MLP}_\text{I}\left( v_\varphi(r) \right).
\label{eq:independ-var}
\end{equation}

For a dependent existential variable $V_D$, as in $F=\exists V_D: r_1(V', V_D)\wedge r_2(V_D, V)$, we first transform $F$ to address the logical constraints posed by $r_2$ and the unknown variable $V$. Specifically, we introduce an independent existential variable $V_I$ and the inverse relation $r_2^{-1}$:
\begin{equation}
    F \Rightarrow_* F^* = \exists V_D, V_I: r_1(V', V_D)\wedge r_2^{-1}(V_I, V_D).
\end{equation}

Having obtained $F^* \Rightarrow_S F_S^* =  r_1\left(V', f_{r_1\wedge r_2^{-1}}(V')\right) \wedge r_2^{-1}\left(f_{r_2^{-1}}(), f_{r_1\wedge r_2^{-1}}(V')\right)$, we then infer $\varphi(V_D)$ by
\begin{equation}
    \varphi(V_D) = \text{MLP}_\text{D}\left( v_\varphi(r_1\otimes V') \oplus v_\varphi(r_2^{-1}\otimes V_I) \right).
\label{eq:depend-var}
\end{equation}

\noindent \textbf{Example 1 (cont.)} \textit{Fig. \ref{fig:LVSA-sample} illustrates the reasoning workflow of our LVSA for the sample complex query. Panel (i) shows the query graph and reasoning order. Panel (ii) depicts the mapping of anchor entities into the complex-valued embedding space. Panel (iii) presents the subsequent reasoning step for the existential variable, highlighting the workflow of the differentiable Skolemization module. Panel (iv) demonstrates the inference of the target variable, which involves bundling operations and fast querying based on the similarity metric.}

\subsection{Logic-Driven Optimization Protocol}
\label{sec:logical-constraint-op}
The preceding sections have detailed the architecture and reasoning workflow of our LVSA. We now introduce the logic-driven optimization protocol for model training. The primary training objective for LVSA is a cross-entropy loss. Given a training batch $\mathcal{B}=\{(Q,a)\}$, the loss enforces higher prediction scores for answer entities $\{a\in A[Q]\}_{(Q,a)\in\mathcal{B}}$ compared to negative samples:
\begin{equation}
\begin{aligned}
\mathcal{L}_{\text{CE}}(\mathcal{B}) &= \frac{1}{|\mathcal{B}|}\sum_{(Q,a)\in\mathcal{B}} H(Q,\mathcal{V}) \\
&= \frac{1}{|\mathcal{B}|}\sum_{(Q,a)\in\mathcal{B}} [- E\left(Q_S[a]\right) + \ln \sum_{e\in \mathcal{V}} \exp \left(E\left(Q_S[e]\right) \right)].
\end{aligned}
\end{equation}

To ensure logical consistency in the neural negator, we impose two key logical properties: \textbf{Satisfiability} and \textbf{Logical Axioms}. The satisfiability condition requires that when $V = t$ makes $Q_\neg \equiv \text{True}$, the negator should infer $\neg r_{N+1}(V'_{N+1},t)$ to be true. In addition, the negator must comply with two fundamental logical axioms: the double negation law ($\neg\neg x \equiv x$) and the contradiction law ($\neg x \wedge x \equiv False$). Letting $q[V]= r_{N+1}(V'_{N+1},V)$, these logical constraints are incorporated into the following loss functions:
\begin{align}
    & \begin{aligned}
    \mathcal{L}_{\text{NS}}(\mathcal{B}_{\neg};\alpha) = - \frac{\alpha}{|\mathcal{B}_{\neg}|}  \sum_{Q_{\neg}\in\mathcal{B}_{\neg}}\ln\sigma\left( E\left( \neg q[t]\right) \right), \\
    \end{aligned}
    \\
    & \begin{aligned}
       \mathcal{L}_{\text{NL}}(\mathcal{B}_{\neg};\beta) =  \frac{\beta}{|\mathcal{B}_{\neg}|}\sum_{Q_{\neg}\in\mathcal{B}_{\neg}} 
 & \text{MSELoss}\left( \varphi\left(\neg\neg q\right), \varphi\left(q\right) \right) + \\
        & \text{MSELoss}\left( \varphi\left(\neg q\right), -\varphi\left(q\right) \right),
    \end{aligned}
\end{align}
where $\sigma(\cdot)$ is the sigmoid function, $\text{MSELoss}(x, y) = \|x - y\|_2^2$ is the mean squared error, and $\alpha$, $\beta$ are weighting coefficients for $\mathcal{L}_{\text{NS}}$ and $\mathcal{L}_{\text{NL}}$, respectively.

To mitigate overfitting, we adopt a three-stage curriculum learning strategy~\cite{bengio2009curriculum} with gradient backpropagation, structured as follows: 
\begin{enumerate}[label=(\roman*)]
\item Train entity and relation embeddings ($\{\varphi(e)|e\in\mathcal{V}\}\cup\{\varphi(r)|r\in\mathcal{R}\}$) on atomic queries (\textbf{\textit{1p}}): $\mathcal{L}_1=\mathcal{L}_\text{CE}(\mathcal{B}_\textbf{\textit{1p}})$.
\item Fix pretrained embeddings and fine-tune the differentiable Skolemization module ($\text{MLP}_\text{I}$ and $\text{MLP}_\text{D}$) on multi-hop queries (\textbf{\textit{2p}} and \textbf{\textit{3p}}): $\mathcal{L}_2=\mathcal{L}_\text{CE}(\mathcal{B}_{\textbf{\textit{2p}}\wedge\textbf{\textit{3p}}})$.
\item Optimize the neural negator ($\text{MLP}_\text{N}$) on queries involving negation (\textbf{\textit{2in}}), with pretrained parameters frozen: $\mathcal{L}_3=\mathcal{L}_\text{CE}(\mathcal{B}_{\textbf{\textit{2in}}}) + \mathcal{L}_\text{NS}(\mathcal{B}_{\textbf{\textit{2in}}};\alpha) + \mathcal{L}_\text{NL}(\mathcal{B}_{\textbf{\textit{2in}}};\beta)$.
\end{enumerate}

\section{Theoretical Analysis}
\label{sec:theory-analysis}
\begin{theorem}[Universal Expressivity for EFO\textsubscript{1}]
LVSA guarantees universal applicability to all EFO\textsubscript{1} queries over incomplete KGs.
\end{theorem}
\begin{proof}
We establish universal expressivity through constructive proof. Let $Q[V_?]$ be an arbitrary EFO\textsubscript{1} query, which can be represented as a DAG $\mathcal{G}_Q = (\mathcal{V}_Q, \mathcal{E}_Q)$ where nodes $\mathcal{V}_Q$ represent (non-)variables and edges $\mathcal{E}_Q$ encode relational projections.
Following the topological ordering of nodes in $\mathcal{G}_Q$, LVSA computes (non-)variable embeddings sequentially.

\noindent\textbf{Base Case:} For anchor entities and independent existential variables, their embeddings are derived by Eqs. \eqref{eq:complex-emb} and \eqref{eq:independ-var}, respectively.

\noindent\textbf{Inductive Step:} For each dependent existential variable $V_D \in \mathcal{V}_Q$ with one predecessor $u$ and one successor $w$, its embedding is computed by Eq. \eqref{eq:depend-var}.

\noindent\textbf{Canonicalization:} $V_D$ with multiple dependencies is reduced to this canonical form through algebraic manipulation. Consider $V_D$ with forward ($F_D$) and backward ($B_D$) dependencies:
\begin{equation}
F_D = \exists V_D: \bigwedge_{i=1}^N r_{F,i}(u_i, V_D),~
B_D = \exists V_D: \bigwedge_{j=1}^M r_{B,j}(V_D, w_j).
\label{eq:vd-dependency}
\end{equation}

Given $B_D\Rightarrow_* B_D^*= \exists V_D: \bigwedge_{m=1}^M r^{-1}_{B,m}(V_{I,m},V_D)$, LVSA unifies them via superposition:
\begin{equation}
\begin{aligned}
    & \varphi(F_D)=\varphi(r_{F,1}\otimes V_{F,1} + \dots + r_{F,N}\otimes V_{F,N}), \\
    & \varphi(B_D^*)=\varphi(r^{-1}_{B,1}\otimes V_{I,1} + \dots + r^{-1}_{B,M}\otimes V_{I,M}), \\
    & \varphi(V_D)=\text{MLP}_{\text{D}}\left(\varphi(F_D)\oplus\varphi(B_D^*)\right).
\end{aligned}
\end{equation}

By structural induction on $\mathcal{G}_Q$'s topology, all EFO\textsubscript{1} constructs are preserved under LVSA's algebraic operations, which enables LVSA to systematically compute embeddings and execute sound logical inference for arbitrary EFO\textsubscript{1} queries.
\end{proof}

\begin{proposition}[Computational Complexity of LVSA for Existential Quantification]
\label{prop:complexity}
Let $Q[V_?]$ be a $k$-hop query and let $d$ be the embedding dimension. The general complexity for Skolemization-based methods is $\mathcal{O}\left(k\cdot cost(g) + cost(\phi_E)\right)$, where $\text{cost}(g)$ is the cost of Skolem function approximation per existential variable and $\text{cost}(\phi_E)$ denotes the cost of fast querying. LVSA achieves an efficient instantiation of this paradigm with $\mathcal{O}\left(k\cdot d^2 + |\mathcal{V}| \cdot d\right)$. This represents an advantage over both Grounding-based methods, which typically exhibit exponential scaling in $k$, and other Skolemization-based methods, which often have higher constant factors in their implementations.
\end{proposition}
\begin{proof}
We analyze the computational complexity through a systematic comparison of different paradigms. 

\noindent\textbf{Complexity Analysis of LVSA:} The computation in LVSA comprises four main steps: (i) relational projection via complex-valued binding operations, requiring $\mathcal{O}(k\cdot d)$ time, (ii) processing of independent existential variables via $\text{MLP}_\text{I}$, (iii) handling of dependent existential variables via $\text{MLP}_\text{D}$, and (iv) final answer retrieval via similarity function $\phi_E$. Assuming fixed-depth MLPs with hidden size $\mathcal{O}(d)$, the costs of both $\mathrm{MLP_I}$ and $\mathrm{MLP_D}$ are $\mathcal{O}(d^2)$. The overall complexity is therefore expressed as $\mathcal{O}(k \cdot (d + cost(\text{MLP}_\text{I}) + cost(\text{MLP}_\text{D})) + cost(\phi_E))=\mathcal{O}(k \cdot d^2 + |\mathcal{V}| \cdot d)$.

\noindent\textbf{Comparison with Other Skolemization-based Methods:}
The Skolem function approximation overhead (i.e., $cost(g)$) is primarily governed by the complexity of the underlying neural architecture. While LVSA and most geometric embedding-based methods employ lightweight MLPs with $cost(g) = \mathcal{O}(d^2)$, GNN-based methods incur substantially higher costs that scale with the structural complexity of the query, typically reaching $\mathcal{O}(k \cdot d^2)$. Regarding the fast querying cost (i.e., $cost(\phi_E)$), this component depends mainly on the computational characteristics of the embedding space. Both LVSA and GNN-based methods using vector embeddings achieve efficient similarity computation, whereas geometric embedding-based methods often rely on specialized geometric transformations that introduce additional computational overhead. Overall, LVSA's architectural simplicity yields notable computational advantages, while other Skolemization-based methods generally involve larger constant factors due to their more complex geometric or architectural mechanisms.

\noindent\textbf{Comparison with Grounding-based Methods:}
Assuming the neural link prediction cost per atomic literal is $cost(\phi) = \mathcal{O}(|\mathcal{V}| \cdot d)$, Grounding-based methods ground each existential variable by enumerating over the entity set $\mathcal{V}$, leading to a complexity of $\mathcal{O}(|\mathcal{V}|^{k}\cdot |\mathcal{V}| \,\cdot\,d)=\mathcal{O}(|\mathcal{V}|^{k+1}\,\cdot\,d)$. CQD-Beam employs beam search to mitigate the exponential dependence on the entity set size $|\mathcal{V}|$, reducing complexity to $\mathcal{O}(b^k \cdot |\mathcal{V}| \cdot d)$ where $b$ is the beam size. However, its complexity remains exponential in the number of existential variables $k$. This fundamental scaling difference demonstrates the superior scalability of the Skolemization-based CQA paradigm and highlights the particular efficiency of LVSA's design.
\end{proof}

\begin{figure*}[t]
  \centering
  \includegraphics[width=1.\textwidth]{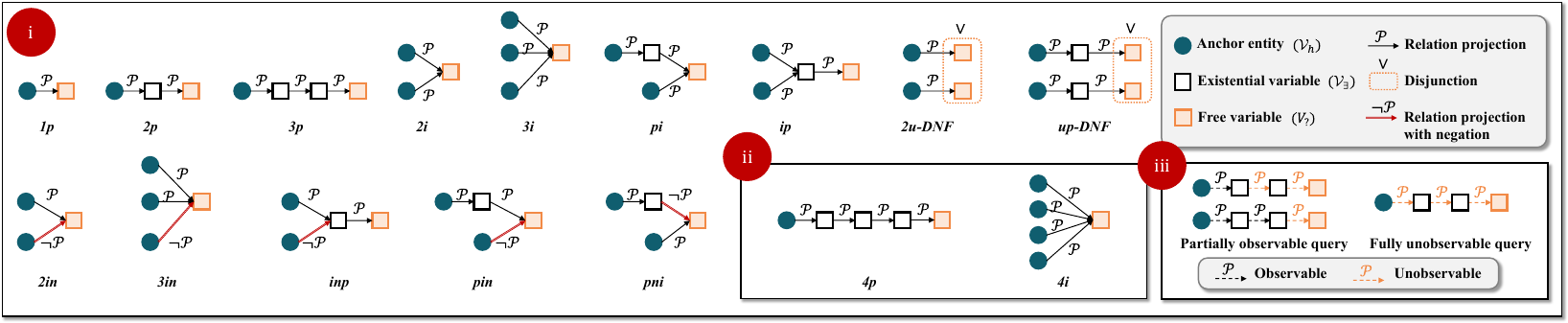}
  \caption{\textbf{(i)} Query graphs of 14 standard query types. Specifically, query types labeled with ``\textit{\textbf{p}}'' (except the atomic query \textit{\textbf{1p}}) contain existential quantifiers. Besides, ``\textit{\textbf{i}}'', ``\textit{\textbf{u}}'', and ``\textit{\textbf{n}}'' indicate the presence of logical conjunction, disjunction, and negation operations, respectively. \textbf{(ii)} Query graphs of novel query types from new benchmarks. \textbf{(iii)} Samples of partially observable queries in the standard benchmarks and fully unobservable queries in the new benchmarks.}
  \label{fig:query-graph}
\end{figure*}

\begin{table*}[t]
  \caption{Statistics of the adopted datasets. - indicates that the corresponding dataset is not provided in the benchmark.}
  \label{tab:statistics}
  \centering
  \resizebox{1.\linewidth}{!}{%
  \begin{tabular}{cccccccccccc||cccc}
    \toprule[1.5pt]
      \multirow{2}{*}[-1.ex]{\textbf{\makecell{Standard\\Datasets  \cite{DBLP:conf/iclr/RenHL20}}}} & \multirow{2}{*}{\textbf{\#Entity}} & \multirow{2}{*}{\textbf{\#Relation}} & & \multicolumn{2}{c}{\textbf{\#Training Query}} & & \multicolumn{2}{c}{\textbf{\#Validation Query}} & & \multicolumn{2}{c||}{\textbf{\#Testing Query}} 
      & \multirow{2}{*}[-1.ex]{\textbf{\makecell{New\\Datasets \cite{greguccicomplex}}}} & \multicolumn{3}{c}{\textbf{\#Testing Query in Total}}
      \\
    \cmidrule{5-6} \cmidrule{8-9} \cmidrule{11-12} \cmidrule{14-16}
      & & & & \textbf{\textit{1p/2p/3p/2in}} & \textbf{Others} & & \textbf{\textit{1p}} & \textbf{Others} & & \textbf{\textit{1p}} & \textbf{Others} 
      & & \textbf{14 Types} & \textbf{\textit{4p}} & \textbf{\textit{4i}}
      \\
    \midrule
    \textbf{FB15k} & 14,951 & 1,345 & & 273,710 & 273,71 & & 59,097 & 8,000 & & 67,016 & 8,000 
    % & \multirow{2}{*}{\textbf{FB15k-237+H}} & \multirow{2}{*}{68,687} & \multirow{2}{*}{2,670} & \multirow{2}{*}{20,189}
    & - & - & - & -
    \\
    \textbf{FB15k-237} & 14,505 & 237 & & 149,689 & 149,68 & & 20,101 & 5,000 & & 22,812 & 5,000 
    % & &
    & {\textbf{FB15k-237+H}} & {68,687} & {2,670} & {20,189}
    \\
    \textbf{NELL995} & 63,361 & 200 & & 107,982 & 107,98 & & 16,927 & 4,000 & & 17,034 & 4,000 
    & \textbf{NELL995+H} & 71,464 & 2,322 & 22,326
    \\
    \bottomrule[1.5pt]
  \end{tabular}
  }
\end{table*}

\section{Experiments}
\label{sec:exp}
\subsection{Experimental Settings}
\label{sec:exp-settings}
\noindent\textbf{Datasets:}
We adopt the standard CQA benchmarks established by Ren et al.~\cite{DBLP:conf/nips/RenL20} and the new benchmarks from Gregucci et al.~\cite{greguccicomplex}, derived from three widely used KGs: FB15k~\cite{DBLP:conf/sigmod/BollackerEPST08}, FB15k-237~\cite{DBLP:conf/acl-cvsc/ToutanovaC15}, and NELL995~\cite{DBLP:conf/emnlp/XiongHW17}, using their official training/validation/testing splits. Detailed statistics are provided in Table \ref{tab:statistics}. Notably, standard benchmarks primarily evaluate performance on partially observable complex queries that can often be guided by 1-hop lookup, while new benchmarks exclusively assess fully unobservable queries requiring compositional reasoning over learned representations. Each benchmark includes 14 basic query types, consisting of 8 types (\textbf{\textit{2p/3p/ip/pi/up/inp/pin/pni}}) with existential quantifiers and 5 types (\textbf{\textit{2in/3in/inp/pin/pni}}) involving logical negation. The new benchmarks further introduce two query types (\textbf{\textit{4p/4i}}). Comprehensive query descriptions are shown in Fig. \ref{fig:query-graph}. 

\begin{table}[t]
    % \centering
    \caption{Grid search configuration and results on FB15k-237's validation set, using average MRR (\%) of \textbf{\textit{2p}}, \textbf{\textit{3p}}, and \textbf{\textit{2i}} queries as selection metric.}
\label{tab:grid-search}
    \centering
    \resizebox{1\linewidth}{!}{%
	\begin{tabular}{c||ccc||ccc||ccc}
    \toprule[1.5pt]
    \cellcolor{mygray}\textbf{Parameter} &
    \multicolumn{3}{c||}{\cellcolor{mygray}\textbf{$\alpha$}} &
    \multicolumn{3}{c||}{\cellcolor{mygray}\textbf{$\beta$}} &
    \multicolumn{3}{c}{\cellcolor{mygray}\textbf{lr}}
    \\
    Search Range &
    0.1 & 1.0 & 10 &
    0.1 & 1.0 & 10 &
    1e-4 & 5e-4 & 1e-3 
    \\
    \midrule
    MRR (\%) &
    \textbf{10.5} & \textbf{10.5} & 10.4 &
    10.1 & \textbf{10.5} & 10.4 &
    10.3 & \textbf{10.5} & 10.4 
    \\
    \bottomrule[1.5pt]
	\end{tabular}
    }
\end{table}

\noindent\textbf{Evaluation Protocol:} Following the evaluation protocol of Ren et al.~\cite{DBLP:conf/nips/RenL20}, we compute the filtered Mean Reciprocal Rank (MRR) and Hits at $K$ (H@$K$) by ranking each test query's non-trivial answers against negative candidates, excluding known true triples from the training/validation sets. Formally, for a complex query $Q[V_?]$ with multiple answers $A[Q]\subseteq \mathcal{V}$, the filtered MRR and H@$K$ are defined by
\begin{align}
   \text{MRR} &= \frac{1}{|A[Q]|} \sum_{a\in A[Q]} \frac{1}{\text{Rank}(a \mid \mathcal{V}\setminus A[Q])}, \\
   \text{H@}K &= \frac{1}{|A[Q]|} \sum_{a\in A[Q]} \textbf{1}\left[ \text{Rank}(a \mid \mathcal{V}\setminus A[Q]) \le K \right],
\end{align}
where $\text{Rank}(a \mid \mathcal{V}\setminus A[Q])$ is the rank of answer $a$ against all non-answer entities $\mathcal{V}\setminus A[Q]$ and $\textbf{1}[\cdot]$ is the indicator function.

\noindent\textbf{Baselines:} We compare LVSA against a comprehensive set of existing CQA approaches. Skolemization-based baselines include CQD-CO~\cite{DBLP:conf/iclr/ArakelyanDMC21}, BetaE~\cite{DBLP:conf/nips/RenL20}, ConE~\cite{DBLP:conf/nips/ZhangWCJW21}, LMPNN~\cite{DBLP:conf/iclr/WangSWS23}, and CLMPT~\cite{zhang2024conditional}. Grounding-based baselines include CQD-Beam~\cite{DBLP:conf/iclr/ArakelyanDMC21}, GNN-QE~\cite{DBLP:conf/icml/Zhu0Z022}, ULTRAQ~\cite{DBLP:conf/nips/GalkinZ00Z24}, and QTO~\cite{DBLP:conf/icml/BaiLLH23}. Following their original papers, all baselines that rely on pretrained knowledge graph embeddings use ComplEx~\cite{trouillon2016complex}, except for the geometric embedding-based methods BetaE and ConE. More discussions of the core algorithms and specific model designs for each baseline can be found in Section \ref{sec:dichotomy}.

\begin{table*}[t]
  \caption{The MRR (\%) results of our LVSA and the existing Skolemization-based baselines on the standard benchmarks. $A_p$, $A_n$, and $A$ are the average scores among EPFO queries, queries involving negation, and all queries, respectively.}
  \label{tab:main-exp-betae}
  \centering
  \resizebox{1.\linewidth}{!}{%
  \begin{tabular}{cccccccccccccccccccc}
    \toprule[1.5pt]
    \textbf{Datasets} & \textbf{Model} & & $A_p$ & $A_n$ & $A$ & \textbf{\textit{1p}} & \textbf{\textit{2p}} & \textbf{\textit{3p}} & \textbf{\textit{2i}} & \textbf{\textit{3i}} & \textbf{\textit{pi}} & \textbf{\textit{ip}} & \textbf{\textit{2u}} & \textbf{\textit{up}} & \textbf{\textit{2in}} & \textbf{\textit{3in}} & \textbf{\textit{inp}} & \textbf{\textit{pin}} & \textbf{\textit{pni}} \\
    \midrule[0.75pt]
    \midrule[0.75pt]
    \multirow{6}{*}{FB15k} & CQD-CO & & 45.3 & 4.6 & 30.8
    & \cellcolor{mygreen}\textbf{89.4} &  27.6 &  15.1 & 63.0 & 65.5 & 46.0 & 35.2 & 42.9 & 23.2
    & 0.2 & 0.2 & 4.0 & 0.1 & 18.4 \\
    & BetaE & & 41.6 & 11.8 & 31.0
    & 65.1 & 25.7 & 24.7 & 55.8 & 66.5 & 43.9 & 28.1 & 40.1 & 25.2
    & 14.3 & 14.7 & 11.5 & 6.5 & 12.4 \\
    & ConE & & 49.8 & 14.8 & 37.3
    & 73.3 &	33.8 &	29.2 &	64.4 &	73.7 &	50.9 &	35.7 &	55.7 &	31.4
    & 17.9 &	18.7 &	12.5 &	9.8 &	15.1 \\
    & LMPNN & & 50.6 & 20.0 & 39.7
    & 85.0 & 39.3 & 28.6 & 68.2 & 76.5 & 46.7 & 43.0 & 36.7 & 31.4
    & 29.1 & 29.4 & 14.9 & 10.2 & 16.4 \\
    & CLMPT & & 55.1 & 17.9 & 41.8
    & 86.1 & \cellcolor{mygreen}\textbf{43.3} & \cellcolor{mygreen}\textbf{33.9} & 69.1 & 78.2 & 55.1 & \cellcolor{mygreen}\textbf{46.6} & 46.1 & \textbf{37.3}
    & 21.2 & 22.9 & \cellcolor{mygreen}\textbf{17.0} & 13.0 & 15.3 \\
    \cmidrule{2-20}
     & \cellcolor{myorange}\textbf{LVSA} & & \cellcolor{mygreen}\textbf{59.9} & \cellcolor{mygreen}\textbf{27.7} & \cellcolor{mygreen}\textbf{48.4}
    & \cellcolor{mygreen}\textbf{89.4} &	40.6 &	32.0 &	\cellcolor{mygreen}\textbf{80.4} &	\cellcolor{mygreen}\textbf{84.2} &	\cellcolor{mygreen}\textbf{61.1} &	43.2 &	\cellcolor{mygreen}\textbf{72.0} &	35.8
    & \cellcolor{mygreen}\textbf{41.5} &	\cellcolor{mygreen}\textbf{33.7} &	15.0 &	\cellcolor{mygreen}\textbf{17.9} &	\cellcolor{mygreen}\textbf{30.4} \\
    \midrule[0.75pt]
    \multirow{6}{*}{FB15k-237} & CQD-CO & & 19.8 & 1.7 & 13.3
    & 46.7 & 10.3 & 6.5 & 23.1 & 29.8 & 22.1 & 16.3 & 14.2 & 8.9
    & 0.2 & 0.2 & 2.1 & 0.1 & 6.1 \\
    & BetaE & & 20.9 & 5.5 & 15.4
    & 39.0 & 10.9 & 10.0 & 28.8 & 42.5 & 22.4 & 12.6 & 12.4 & 9.7
    & 5.1 & 7.9 & 7.4 & 3.5 & 3.4 \\
    & ConE & & 23.4 & 5.9 & 17.1
    & 41.8 &	12.8 &	11.0 &	32.6 &	47.3 &	25.5 &	14.0 &	14.5 &	10.8
    & 5.4 &	8.6 &	7.8 &	4.0 &	3.6 \\
    % & FuzzQE & 24.0 & 7.8
    % & 42.8 & 12.9 & 10.3 & 33.3 & 46.9 & 26.9 & 17.8 & 14.6 & 10.3
    % & 8.5 & 11.6 & 7.8 & 5.2 & 5.8 \\
    & LMPNN & & 24.1 & 7.8 & 18.3
    & 45.9 & 13.1 & 10.3 & 34.8 & 48.9 & 22.7 & 17.6 & 13.5 & 10.3
    & 8.7 & 12.9 & 7.7 & 4.6 & 5.2 \\
    & CLMPT & & 25.9 & 7.9 & 19.4
    & 45.7 & 13.7 & 11.3 & 37.4 & 52.0 & 28.2 & \textbf{19.0} & 14.3 & 11.1
    & 7.7 & 13.7 & 8.0 & 5.0 & 5.1 \\
    \cmidrule{2-20}
    \rowcolor{mygreen}\cellcolor{white} & \cellcolor{myorange}\textbf{LVSA} & \cellcolor{white} & \textbf{28.8} & \textbf{10.9} & \textbf{22.4}
    & \textbf{49.2} &	\textbf{15.8} &	\textbf{12.1} &	\textbf{41.7} &	\textbf{55.1} &	\textbf{31.0} &	\textbf{19.0} &	\textbf{22.5} &	\textbf{12.4}
    & \textbf{12.4} &	\textbf{18.3} &	\textbf{8.5} &	\textbf{7.3} &	\textbf{8.0} \\
    \midrule[0.75pt]
    \multirow{6}{*}{NELL995} & CQD-CO & & 28.4 & 1.9 & 18.9
    & 60.8 & 18.3 & 13.2 & 36.5 & 43.0 & 30.0 & 22.5 & 17.6 & 13.7
    & 0.1 & 0.1 & 4.0 & 0.0 & 5.2 \\
    & BetaE & & 24.6 & 5.9 & 17.9
    & 53.0 & 13.0 & 11.4 & 37.6 & 47.5 & 24.1 & 14.3 & 12.2 & 8.5
    & 5.1 & 7.8 & 10.0 & 3.1 & 3.5 \\
    % & FuzzQE & 27.0 & 7.8
    % & 47.4 & 17.2 & 14.6 & 39.5 & 49.2 & 26.2 & 20.6 & 15.3 & 12.6
    % & 7.8 & 9.8 & 11.1 & 4.9 & 5.5 \\
    & ConE & & 27.1 & 6.4 & 19.7
    & 53.1 &	16.1 &	13.9 &	40.0 &	50.8 &	26.3 &	17.5 &	15.3 &	11.3
    & 5.7 &	8.1 &	10.8 &	3.5 &	3.9 \\
    & LMPNN & & 30.7 & 8.0 & 22.6
    & 60.6 & 22.1 & 17.5 & 40.1 & 50.3 & 28.4 & 24.9 & 17.2 & 15.7
    & 8.5 & 10.8 & \cellcolor{mygreen}\textbf{12.2} & 3.9 & 4.8 \\
    & CLMPT & & 31.3 & 7.0 & 22.6
    & 58.9 & 22.1 & 18.4 & 41.8 & 51.9 & 28.8 & 24.4 & 18.6 & 16.2
    & 6.6 & 8.1 & 11.8 & 3.8 & 4.5 \\
    \cmidrule{2-20}
    \rowcolor{mygreen}\cellcolor{white} & \cellcolor{myorange}\textbf{LVSA} & \cellcolor{white} & \textbf{32.4} & \textbf{9.1} & \textbf{24.1}
    & \textbf{61.0} &	\textbf{22.2} &	\textbf{18.9} &	\textbf{43.6} &	\textbf{53.2} &	\textbf{31.8} &	\textbf{25.0} &	\textbf{19.4} &	\textbf{16.6}
    & \textbf{10.1} &	\textbf{11.4} &	\cellcolor{white}11.8 &	\textbf{6.1} &	\textbf{6.1} \\
    \bottomrule[1.5pt]
  \end{tabular}
  }
\end{table*}

\noindent\textbf{Implementation:} All experiments are conducted on a single Nvidia RTX 3090 GPU with 24GB of memory. LVSA uses LeakyReLU activations for neural modules. Based on grid search results (Table \ref{tab:grid-search}), we set regularization weights to $\alpha=\beta=1$. For each training stage, we employ the Adam optimizer with an initial learning rate of 5e-4. Following Chen et al.~\cite{DBLP:conf/akbc/ChenM0S21}, we incorporate both entity prediction and relation prediction objectives into the primary loss function.

As detailed in Section \ref{sec:basic-symbolic}, LVSA aligns with ComplEx in the designs of the embedding space and the triple scoring metric. This allows initializing the entity and relation embeddings in LVSA using pre-trained ComplEx embeddings ($d=1000$), with subsequent optimization limited to the differentiable Skolemization module and the neural negator. To mitigate the risk of local optima inherent to Skolemization's satisfiability relaxation (Property \ref{property:skolem}a), we introduce a cross-entropy term over the in-batch candidate sets $\mathcal{V}_\mathcal{B}=\{a \in A[Q]\}_{(Q,a)\in\mathcal{B}}$ to stabilize the training of the differentiable Skolemization module, which is defined as follows:
\begin{equation}
    \mathcal{L}_\text{CE}(\mathcal{B}_{2p\wedge 3p})=\frac{1}{|\mathcal{B}|}\sum_{Q\in\mathcal{B}} H(Q,\mathcal{V}) + H(Q,\mathcal{V}_\mathcal{B}).
\end{equation}

\begin{table}[t]
    % \centering
    \caption{The H@1 (\%) results of our LVSA and Skolemization-based baselines on the standard benchmarks.}
\label{tab:tab:main-exp-betae-h1}
    \centering
    \resizebox{1\linewidth}{!}{%
	\begin{tabular}{cccccccccccccc}
    \toprule[1.5pt]
    \multirow{2}{*}{\textbf{Model}} & &
    \multicolumn{3}{c}{\textbf{FB15k}} & &
    \multicolumn{3}{c}{\textbf{FB15k-237}} & &
    \multicolumn{3}{c}{\textbf{NELL995}}
    \\
    \cmidrule[0.75pt]{3-5} \cmidrule[0.75pt]{7-9} \cmidrule[0.75pt]{11-13} 
     & &
    $A_p$ & $A_n$ & $A$ & &
    $A_p$ & $A_n$ & $A$ & &
    $A_p$ & $A_n$ & $A$
    \\
    \midrule[0.75pt]
    \midrule[0.75pt]
    CQD-CO & & 39.7 & 1.9 & 26.2 & &
    14.7 & 0.5 & 9.7 & &
    21.3 & 0.7 & 14.0
    \\
    BetaE & & 31.3 & 5.2 & 21.9 & &
    13.4 & 2.8 & 9.6 & &
    17.8 & 2.1 & 12.2
    \\
    ConE & & 39.6 & 7.3 & 28.0 & &
    15.6 & 2.2 & 10.9 & &
    19.8 & 2.2 & 13.5
    \\
    LMPNN & & 41.4 & 12.2 & 31.0 & &
    16.3 & 3.5 & 11.7 & &
    23.0 & 3.4 & 16.0
    \\
    CLMPT & & 45.9 & 9.3 & 32.9 & &
    17.4 & 2.9 & 12.2 & &
    22.7 & 2.4 & 15.4
    \\
    \midrule[0.75pt]
    \rowcolor{mygreen}\cellcolor{myorange}\textbf{LVSA} & \cellcolor{white} & \textbf{53.2} & \textbf{18.9} & \textbf{41.0} & &
    \textbf{20.8} & \textbf{5.5} & \textbf{15.4} & &
    \textbf{24.3} & \textbf{3.8} & \textbf{17.0}
    \\
    \bottomrule[1.5pt]
	\end{tabular}
    }
\end{table}

\subsection{Main Results}
\label{sec:main-exp}
Our main experiments evaluate the proposed LVSA against Skolemization-based baselines on standard benchmarks. The results in Tables \ref{tab:main-exp-betae} and \ref{tab:tab:main-exp-betae-h1} reveal distinct limitations in existing approaches that LVSA is designed to address.

CQD-CO~\cite{DBLP:conf/iclr/ArakelyanDMC21} exhibits limited performance across all query types, particularly on multi-hop queries like \textbf{\textit{3p}}, due to its oversimplified Skolem function approximation. BetaE~\cite{DBLP:conf/nips/RenL20} and ConE~\cite{DBLP:conf/nips/ZhangWCJW21} struggle with queries involving logical negations, revealing their fundamental difficulty in harmonizing geometric properties with logical constraints. Although GNN-based methods like LMPNN~\cite{DBLP:conf/iclr/WangSWS23} and CLMPT~\cite{zhang2024conditional} achieve competitive results on certain queries, their performance varies substantially across datasets. For instance, CLMPT's strong performance on FB15k's existential queries like \textbf{\textit{2p}} and \textbf{\textit{3p}}, contrasted with its weaker results on the cleaned FB15k-237, suggests a reliance on memorizing dataset-specific patterns rather than robust reasoning. More critically, their black-box nature prevents tracing reasoning errors, whereas LVSA's transparent computation graph enables full traceability, as validated in Section \ref{sec:exp-skolem} later.

LVSA demonstrates statistically significant superiority over all baselines (paired t-test, \textit{p} $<$ 0.01). Specifically, the significance test was performed by treating the MRR scores of LVSA and each baseline on 14 query types as paired data arrays. LVSA achieves particularly strong performance on conjunction-heavy queries like \textbf{\textit{2i}} (MRR: 41.7\% vs. CLMPT's 37.4\% on FB15k-237) and queries involving logical negation like \textbf{\textit{2in}} (MRR: 12.4\% vs. LMPNN's 8.7\% on FB15k-237). These results empirically validate the effectiveness of LVSA's binding and bundling operations for logical conjunction and the logically constrained neural negator. Although LVSA slightly underperforms CLMPT on FB15k's \textbf{\textit{2p}} and \textbf{\textit{3p}} queries, it achieves robust reasoning performance across all three standard datasets. This performance difference may stem from the known inverse relation leakage in FB15k~\cite{DBLP:conf/acl-cvsc/ToutanovaC15}. While CLMPT benefits from such dataset-specific artifacts, LVSA maintains principled reasoning patterns that promise better generalization to real-world KGs unaffected by similar biases. 

However, LVSA exhibits diminished performance advantages on \textbf{\textit{ip}} and \textbf{\textit{inp}} queries when superposition is used as input to the differentiable Skolemization module. Given its strong performance on \textbf{\textit{2i}} and \textbf{\textit{3i}} queries, which demonstrate the effectiveness of the bundling mechanism, we attribute this performance drop to the module's current limitations in handling specific input representations. Therefore, enhancing the robustness of the differentiable Skolemization module represents an important direction for future work.

\begin{figure*}[t]
\begin{minipage}[t]{0.70\linewidth}
\vspace{0pt}
\captionof{table}{The MRR (\%) results on new benchmarks, with relative decline ($\triangle$) from standard benchmarks. $A_{p,H}$, $A_{n,H}$, and $A_H$ represent the average scores across EPFO, negation, and all 14 standard query types on the new benchmarks, respectively.} 
\centering
\resizebox{0.98\linewidth}{!}{%
  \begin{tabular}{cccccccccccccc}
    \toprule[1.5pt]
    \multirow{2}{*}{\textbf{Paradigm}} & \multirow{2}{*}{\textbf{Model}} &  &
    \multicolumn{5}{c}{\textbf{FB15k-237+H}} 
    & &
    \multicolumn{5}{c}{\textbf{NELL995+H}}
    \\
    \cmidrule[0.75pt]{4-8} \cmidrule[0.75pt]{10-14}
    & & &
    $A_{p,H}$ & $A_{n,H}$ & $A_H$ & $A$ & \cellcolor{mygray}$\triangle (\downarrow)$
    & & 
    $A_{p,H}$ & $A_{n,H}$ & $A_H$ & $A$ & \cellcolor{mygray}$\triangle (\downarrow)$
    \\
    \midrule[0.75pt]
    \midrule[0.75pt]
    \multirow{4}{*}{\makecell{Grounding \\ -based}} & GNN-QE & 
    & 12.6	& 3.1
    & 9.2 & 20.9 & 55.9 
    & &
    16.9	& 2.7
    & 11.9 & 22.1 & 46.3
    \\
    & ULTRAQ & 
    & 11.4	& 2.4
    & 8.2 & 18.4 & \cellcolor{mygreen}\textbf{55.5}
    & &
    11.8 & 2.3
    & 8.4 & 17.1 & 50.9
    \\
    \rowcolor{mygreen}\cellcolor{white}-based & \cellcolor{white} QTO & \cellcolor{white}
    & \textbf{12.7}	& \textbf{4.2}
    & \textbf{9.7} & \textbf{27.1} & \cellcolor{white}64.3 
    & &
    \textbf{20.2} & \textbf{5.0}
    & \textbf{14.8} & \textbf{25.7} & \textbf{42.6}
    \\
    \cmidrule{2-14}
    \rowcolor{mygray}\cellcolor{white} & \textit{Avg.} & \cellcolor{white}
    & \textit{12.2}	& \textit{3.2}
    & \textit{9.0} & \textit{22.1} & \textit{58.6} 
    & &
    \textit{16.3} & \textit{3.3}
    & \textit{11.7} & \textit{21.6} & \textit{46.6}
    \\
    \midrule[0.75pt]
    \multirow{4}{*}{\makecell{Skolemization\\-based}} & ConE & 
    & 11.6	& 2.8
    & 8.5 & 17.1 & 50.5 
    & &
    19.6	& 3.2
    & 13.8 & 19.7 & 30.2
    \\
    & CLMPT &
    & 15.6	& 3.5
    & 11.3 & 19.4 & \cellcolor{mygreen}\textbf{41.9} 
    & &
    24.1	& 3.4
    & 16.7 & 22.6 & 25.9
    \\
    \rowcolor{mygreen}\cellcolor{white}-based  & \cellcolor{myorange}\textbf{LVSA} & \cellcolor{white}
    & \textbf{15.7}	& \textbf{5.8}
    & \textbf{12.2} & \textbf{22.4} & \cellcolor{white}45.6 
    & &
    \textbf{25.5}	& \textbf{5.0}
    & \textbf{18.2} & \textbf{24.1} & \textbf{24.6}
    \\
    \cmidrule{2-14}
    \rowcolor{mygray}\cellcolor{white} & \textit{Avg.} & \cellcolor{white}
    & \textit{14.3}	& \textit{4.0}
    & \textit{10.7} & \textit{19.6} & \textit{46.0} 
    & &
    \textit{23.1} & \textit{3.9}
    & \textit{16.2} & \textit{22.1} & \textit{26.9}
    \\
    \bottomrule[1.5pt]
    \multicolumn{13}{l}{\small * The performance decline $\triangle\%$ is defined as $(A - A_H) / A \times 100\%$.}\\
  \end{tabular}
  }
\label{tab:std-new-delta}
\end{minipage}\hspace{10pt}
\begin{minipage}[t]{0.27\linewidth}
\vspace{0pt}
        \centering
		\includegraphics[width=0.96\columnwidth]{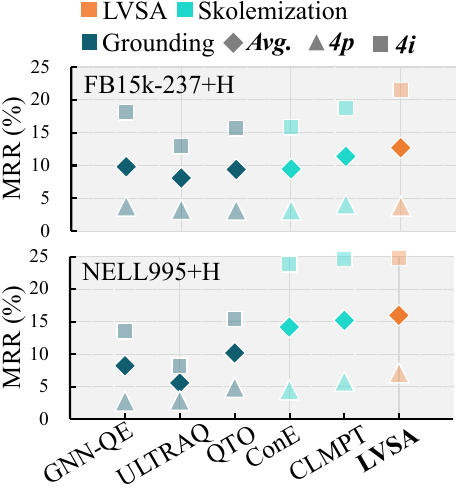}
		\caption{The MRR (\%) results of \textit{\textbf{4p}} and \textit{\textbf{4i}} queries.}
  \label{fig:4p-4i}
\end{minipage}
\vspace{-10pt}
\end{figure*}

\subsection{Results on New Benchmarks}
\label{sec:new-benchmark-results}
The new benchmarks extend the standard evaluation by targeting fully unobservable queries (Fig. \ref{fig:query-graph}(iii)) and introducing more complex query types (Fig. \ref{fig:query-graph}(ii)). We evaluate our LVSA against representative Grounding-based and Skolemization-based methods on these benchmarks to gain deeper insights into model generalization under challenging settings.

\noindent\textbf{Evaluation on Fully Unobservable Queries:} 
Table~\ref{tab:std-new-delta} reveals a notable paradigm shift in performance on fully unobservable queries. While Skolemization-based methods achieve marginally lower reasoning accuracy than Grounding-based methods under partial observability ($A$: 19.6\% vs. 22.1\% on FB15k-237), they attain higher MRR on fully unobservable queries ($A_H$: 10.7\% vs. 9.0\% on FB15k-237+H) and exhibit a significantly smaller performance drop when transitioning to new benchmarks ($\triangle$: 46.0\% vs. 58.6\% from FB15k-237 to FB15k-237+H). Among all evaluated models, LVSA not only achieves the highest MRR on unobservable queries but also maintains a relatively small performance decline, confirming its effectiveness in this challenging scenario.

We propose that the performance disparity stems from the fundamental operational differences between the two paradigms. Grounding-based methods ensure completeness through exhaustive search on the entity set, enabling strong performance on partially observable queries by fully leveraging memorized facts. However, Skolemization-based methods seek to find at least one valid reasoning path via approximated Skolem functions. This objective inherently necessitates selective access to relevant facts and compositional reasoning, reducing reliance on exhaustive memorization and enhancing robustness in real-world settings with sparse prior knowledge. The superior performance of LVSA further demonstrates that its simple architecture and transparent reasoning process mitigate overfitting to observable shortcuts and improve generalization under uncertainty.

\noindent\textbf{Evaluation on More Complex Query Structures:} Consistent with results on the 14 standard query types, Skolemization-based methods outperform Grounding-based approaches on the more complex \textbf{\textit{4i}} and \textbf{\textit{4p}} queries from the new benchmarks, with LVSA maintaining a leading position (Fig. \ref{fig:4p-4i}). These results further validate the effectiveness of LVSA's differentiable Skolemization module and bundling mechanism. 

\subsection{Efficiency Analysis}
\label{sec:efficiency-exp}
To complement the theoretical analysis of computational complexity (Proposition \ref{prop:complexity}), this section empirically evaluates the efficiency of LVSA against representative baselines, including the Skolemization-based methods CLMPT and ConE, and the tractable Grounding-based method CQD-Beam~\cite{DBLP:conf/iclr/ArakelyanDMC21}. Fig.~\ref{fig:efficiency} compares the average MRR and Queries Per Second (QPS) across these models on existential positive queries from FB15k-237. The results reveal a clear performance-efficiency tradeoff between the geometric embedding-based ConE and the GNN-based CLMPT. As a contrast, LVSA successfully breaks this tradeoff by outperforming the strongest baseline CLMPT, while also delivering 2.6× and 1.7× higher throughput than CLMPT and ConE, respectively. This efficiency is attributable to LVSA's lightweight VSA-based framework, which avoids the high computational demands of CLMPT's GNN+Transformer architecture and ConE's intricate geometric operations. In the Grounding-based paradigm, although reducing the beam size can improve the efficiency of CQD-Beam, it does so at the cost of logical completeness and a significant degradation in reasoning performance. When the beam size $b \geq 128$, CQD-Beam's average throughput falls below that of all evaluated Skolemization-based methods. Notably, while LVSA matches the MRR of CQD-Beam with $64 \le b \le 256$, it provides a 2–15× improvement in QPS.

We further investigate the impact of the number of existential variables ($k$) on reasoning efficiency. As shown in Fig.~\ref{fig:efficiency-p}, when moving from \textbf{\textit{2p}} to \textbf{\textit{3p}} queries, LVSA and ConE exhibit only a marginal reduction in QPS, which is significantly lower than the 55.4\% reduction seen in CLMPT. In contrast, CQD-Beam faces more severe efficiency degradation as $k$ increases. Linear regression trendlines fitted to CQD-Beam's QPS-beam size relationship indicate that its QPS declines nearly ten times faster for \textbf{\textit{3p}} queries than for \textbf{\textit{2p}} queries as the beam size grows. Moreover, CQD-Beam encounters Out-Of-Memory (OOM) errors when the beam size $b \geq 512$ and the number of existential variables $k > 2$, underscoring its scalability limitations. These results collectively affirm that the Skolemization-based paradigm inherently enhances reasoning efficiency in CQA, with LVSA serving as an effective solution that fully realizes the theoretical benefits of this approach.

\begin{figure}[t]
\centering
\includegraphics[width=1\linewidth]{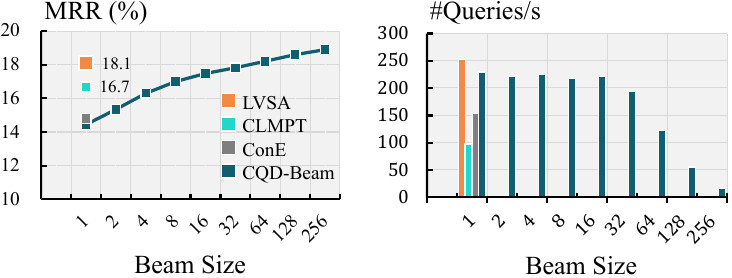}
\caption{The performance (left) and throughput (right) of LVSA, CLMPT, ConE, and CQD-Beam on existential positive queries (\textbf{\textit{2p/3p/ip/pi/up}}) in FB15k-237.}
\label{fig:efficiency}
\end{figure}

\begin{figure}[t]
\centering
\includegraphics[width=1\linewidth]{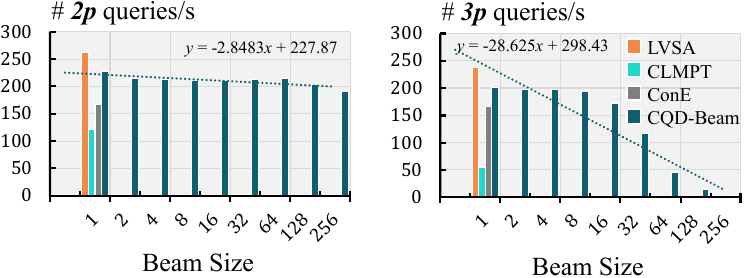}
\caption{The throughput of LVSA and baseline models on \textbf{\textit{2p}} (left) and \textbf{\textit{3p}} (right) queries in FB15k-237. The dashed line shows CQD-Beam's performance decline as beam size grows.}
\label{fig:efficiency-p}
\end{figure}

\subsection{Ablation Analysis}
The architecture of LVSA integrates several core components that operate collaboratively to answer complex queries, including the VSA-based compositional framework, the differentiable Skolemization module, and the neural negator. Within this foundational design, the logical constraints applied to the neural negator constitute the only additional supervisory mechanism. Our ablation study, therefore, focuses on quantifying the contribution of these logic-driven regularizations.

As shown in Table~\ref{tab:ablation-exp}, removing the satisfiability loss $\mathcal{L}_\text{NS}$ leads to a slight performance drop, indicating its role in preventing the model from generating semantically invalid negations. In contrast, ablating the logical axioms loss $\mathcal{L}_\text{NL}$, which enforces the laws of double negation and contradiction, causes the most substantial performance degradation. This is evidenced by MRR reductions of 5.0\%, 3.3\%, and 2.1\% across the three standard benchmarks. Consistent trends are observed on the two new benchmarks. These results confirm that explicitly modeling logical axioms through $\mathcal{L}_\text{NL}$ is essential for handling negation.
This analysis underscores LVSA's distinctive advantage in harmonizing neural and symbolic reasoning. While geometric embedding-based methods such as BetaE and ConE rely on geometric properties to implicitly encourage logical consistency, and GNN-based approaches like LMPNN and CLMPT lack a transparent mechanism for injecting logical rules, LVSA's constraints directly and differentiably optimize toward logical correctness. This principled integration of explicit logical guidance serves as the key to LVSA's superior and more reliable performance.

\begin{table}[t]
    % \centering
    \caption{The MRR (\%) results of our LVSA and ablation models across three standard datasets.}
\label{tab:ablation-exp}
    \centering
    \resizebox{1\linewidth}{!}{%
	\begin{tabular}{clcccccc}
    \toprule[1.5pt]
    \textbf{Datasets} & \textbf{Model} & $A_n$ & \textbf{\textit{2in}} & \textbf{\textit{3in}} & \textbf{\textit{inp}} & \textbf{\textit{pin}} & \textbf{\textit{pni}} \\
	\midrule[0.75pt]
    \midrule[0.75pt]
    \rowcolor{mygray}\cellcolor{white} \multirow{3}{*}{FB15k} &
    LVSA & 27.7 & 41.5 &	33.7 &	15.0 &	17.9 &	30.4 \\
    & w/o $\mathcal{L}_\text{NS}$ & 25.6 & 39.0 & 29.9 &	14.0 &	16.5 &	28.7 \\
    & w/o $\mathcal{L}_\text{NL}$ & 22.7 & 24.6 &	30.1 &	12.5 &	13.1 &	19.0 \\
    \midrule[0.75pt]
    \rowcolor{mygray}\cellcolor{white} \multirow{3}{*}{\makecell{\cellcolor{white}FB15k \\ \cellcolor{white}-237}} &
    LVSA & 10.9 & 12.4 &	18.3 &	8.5 &	7.3 &	8.0 \\
    & w/o $\mathcal{L}_\text{NS}$ & 10.4 & 12.2 &	16.5 &	8.1 &	6.9 &	8.4 \\
    & w/o $\mathcal{L}_\text{NL}$ & 7.6 & 7.4 &	13.9 &	7.0 &	5.1 &	4.5 \\
    \midrule[0.75pt]
    \rowcolor{mygray}\cellcolor{white} \multirow{3}{*}{NELL995} &
    LVSA & 9.1 & 10.1 &	11.4 &	11.8 &	6.1 &	6.1 \\
    & w/o $\mathcal{L}_\text{NS}$ & 9.0 & 9.9 &	11.1 &	11.8 &	6.1 &	6.2 \\
    & w/o $\mathcal{L}_\text{NL}$ & 7.0 & 6.6 &	10.8 &	9.4 &	3.9 &	4.4 \\
    \bottomrule[1.5pt]
	\end{tabular}
    }
\end{table}

\begin{figure*}[t]
\begin{minipage}[t]{0.23\linewidth}
\vspace{0pt}
        \centering
		\includegraphics[width=0.94\columnwidth]{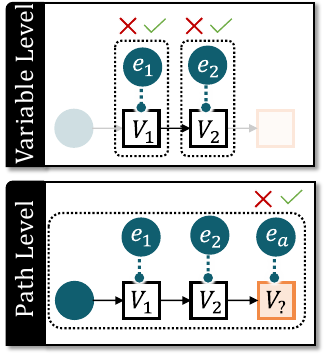}
		\caption{Evaluation at variable and path levels.}
		\label{fig:interpretability}
\end{minipage}\hspace{7pt}
\begin{minipage}[t]{0.75\linewidth}
\vspace{0pt}
\captionof{table}{The results (\%) of quantitative interpretability evaluation on multi-hop queries.}
\centering
\resizebox{0.98\linewidth}{!}{%
  \begin{tabular}{cccccccccccccccccccc}
    \toprule[1.5pt]
      &  & \multicolumn{7}{c}{\cellcolor{mygray}\textbf{Variable-level MRR}} & &  \multicolumn{6}{c}{\cellcolor{mygray}\textbf{Variable-level P@1}} & & \multicolumn{3}{c}{\cellcolor{mygray}\textbf{Path-level P@1}} \\
     \multirow{2}{*}{\textbf{Datasets}} & \multirow{2}{*}{\textbf{Model}} & &  \textbf{\textit{2p}} & & \multicolumn{2}{c}{\textbf{\textit{3p}}} & & \multirow{2}{*}{\textbf{\textit{Avg.}}} & & \textbf{\textit{2p}} & & \multicolumn{2}{c}{\textbf{\textit{3p}}} & & \multirow{2}{*}{\textbf{\textit{Avg.}}} & & \multirow{2}{*}{\textbf{\textit{2p}}} & \multirow{2}{*}{\textbf{\textit{3p}} }& \multirow{2}{*}{\textbf{\textit{Avg.}}} \\
     \cmidrule[0.5pt]{4-4} \cmidrule[0.5pt]{6-7} \cmidrule[0.5pt]{11-11} \cmidrule[0.5pt]{13-14}
     & & & $V_1$ & & $V_1$ & $V_2$ &  & & & $V_1$ & & $V_1$ & $V_2$ &  & &   \\
    \midrule
    \midrule
    \rowcolor{mygreen}\cellcolor{white}\multirow{3}{*}{FB15k} & \cellcolor{myorange}LVSA & \cellcolor{white} & \textbf{57.8} &
    & \cellcolor{white}60.2 & \textbf{40.1} 
    & & \textbf{52.7}
    & & \textbf{50.9} &
    % & 74.8 & \textbf{53.8} & 64.3
    & \cellcolor{white}59.4 & \textbf{33.2} 
    & & \textbf{47.8}
    & & \textbf{69.9}
    & \textbf{40.7} 
    & \textbf{55.3} \\
     & CLMPT & & 53.3 &
     & \cellcolor{mygreen}\textbf{63.5} & 35.3 
     & & 50.7
     & & 45.9 &
     % & \textbf{79.0} & 50.1 & \textbf{64.6}
     & \cellcolor{mygreen}\textbf{61.2} & 28.1 
     & & 45.1
    & & 62.4
     & 38.0 
     & 50.2 \\
     & ConE & & 44.2 &
     & 47.3 & 24.2
     & & 35.8
     & & 35.6
     & & 43.0 & 17.4 &
     & 30.2
     & & 0.8
     & 0.1
     & 0.5
     \\
    \midrule
    \rowcolor{mygreen}\cellcolor{white}\multirow{3}{*}{FB15k-237} & \cellcolor{myorange}LVSA & \cellcolor{white} & \textbf{47.6} &
    & \textbf{52.4} & \textbf{31.7} 
    & & \textbf{43.9}
    & & \textbf{40.3} &
    & \textbf{51.3} & \textbf{24.9} 
    & & \textbf{38.8}
    % & 73.8 & 59.2 & 66.5
    & & \textbf{66.1}
    & \textbf{34.4} 
    & \textbf{50.3} \\
    & CLMPT & & 40.8 &
    & 49.3 & 23.8 
    & & 38.0 
    & & 31.9 &
    & 46.9 & 16.7 
    & & 31.8
    % & 66.1 & 38.4 & 52.3
    & & 41.2
    & 19.0 
    & 30.1 \\
    & ConE & & 34.1 &
    & 38.9 & 16.2
    & & 27.6
    & & 25.2
    & & 33.1 & 10.2
    & & 21.7
    & & 0.7 & 0.1 & 0.4
     \\
    \midrule
    \rowcolor{mygreen}\cellcolor{white}\multirow{3}{*}{NELL995} & \cellcolor{myorange}LVSA & \cellcolor{white} & \textbf{68.3} &
    & \cellcolor{white}61.7 & \textbf{47.1} 
    & & \textbf{59.0} 
    & & \textbf{63.4} &
    & \textbf{61.9} & \textbf{41.2} 
    & & \textbf{55.5}
    % & 73.8 & 59.2 & 66.5
    & & \textbf{83.4}
    & \textbf{57.6} 
    & \textbf{70.5} \\
    & CLMPT & & 63.8 &
    & \cellcolor{mygreen}\textbf{62.3} & 45.1 
    & & 57.1 
    & & 56.8 &
    & 61.7 & 35.6 
    & & 51.4
    % & 77.3 & 63.5 & 70.4
    & & 66.2
    & 38.9 
    & 52.6 \\
    & ConE & & 50.4
    & & 42.0 & 19.6
    & & 30.8
    & & 42.8
    & & 37.5 & 12.3
    & & 24.9
    & & 1.1 & 0.2 & 0.7
     \\
    \bottomrule[1.5pt]
  \end{tabular}
  }
% \vspace{10pt}
\label{tab:exp-interpretability}
\end{minipage}
% \vspace{-10pt}
\end{figure*}

\subsection{Analysis of Differentiable Skolemization Module}
\label{sec:exp-skolem}
This section provides an in-depth analysis of the effectiveness and interpretability of LVSA's differentiable Skolemization module. In an ideal Skolemization-based model, the approximated Skolem functions should identify correct intermediate entities to establish reliable reasoning paths, thereby providing evidential support for the final predictions.

\noindent\textbf{Quantitative Evaluation:} 
To quantitatively assess interpretability, we propose a two-level evaluation framework, illustrated in Fig. \ref{fig:interpretability} using a \textbf{\textit{3p}} query $Q[V_?]=V_?: \exists V_1,V_2: r_1(h,V_1) \wedge r_2(V_1,V_2) \wedge r_3(V_2,V_?)$ as an example. At the variable level, we evaluate whether each existential variable is grounded to a valid intermediate entity. For instance, if $V_1$ is grounded to $e_1$, we verify the satisfiability of the resulting sub-query $\exists V_2, V_?: r_1(h,e_1) \wedge r_2(e_1,V_2) \wedge r_3(V_2,V_?)$. At the path level, we assess whether the fully grounded path, such as $r_1(h,e_1) \wedge r_2(e_1,e_2) \wedge r_3(e_2,e_a)$, forms a valid reasoning chain. Since Skolemization-based methods aim to find at least one solution for existential variables, we employ Precision at $K$ (P@$K$) as a complementary metric to MRR. P@$K$  measures the proportion of correct answers among the top-$K$ predictions: $\text{P@}K = \frac{1}{K} \sum_{i=1}^{K} \mathbf{1}[R_i \in A[Q]]$, where $R_i$ denotes the $i^\text{th}$ answer retrieved by the model.

As shown in Table \ref{tab:exp-interpretability}, LVSA demonstrates superior overall performance compared to the strongest Skolemization-based baseline CLMPT across both evaluation levels. In contrast, the geometric embedding-based method ConE shows poor performance, particularly in generating interpretable reasoning paths. Notably, LVSA's advantage is particularly pronounced in the P@1 metric, indicating its stronger capability in pinpointing correct intermediate answers, whereas CLMPT tends to generate ambiguous intermediate predictions. Although CLMPT achieves performance comparable to LVSA when grounding the first existential variable ($V_1$) in \textbf{\textit{3p}} queries, its performance declines sharply on subsequent variables ($V_2$) and complete reasoning paths. These results suggest that CLMPT's strategy of generating vague intermediate predictions while performing global reasoning over the entire query graph compromises both interpretability and reliability, potentially leading to shortcut reasoning. In contrast, LVSA's substantial advantage at the path level demonstrates its ability to sustain robust multi-hop reasoning by sequentially building on previously inferred results, thereby ensuring both reliable interpretability and causal traceability.

\begin{figure}[t]
\centering
\includegraphics[width=1.\linewidth]{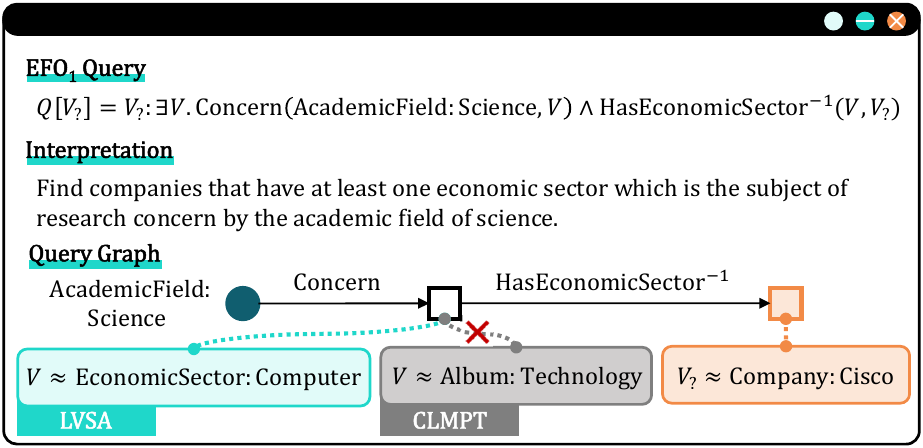}
\caption{A sample of grounding the approximated Skolem function from our LVSA and the baseline of CLMPT.}
\label{fig:case-study}
\end{figure}

\noindent\textbf{Qualitative Evaluation:} A representative case study is shown in Fig. \ref{fig:case-study}, which compares how LVSA and CLMPT ground existential variable embeddings into concrete entities using the similarity metric $V\approx \arg\max_{e\in\mathcal{V}}\phi_E\left( \varphi(V), \varphi(e) \right)$. In this example, LVSA correctly grounds the existential variable to Computer (a science-aligned economic sector), while CLMPT erroneously associates it with a music album. Although both models ultimately arrive at the correct answer Cisco, CLMPT's flawed reasoning process undermines its reliability.

\section{Conclusion}
\label{sec:conclusion}
This work introduces the Grounding–Skolemization dichotomy, a formal framework that systematically exposes the inherent trade-offs in existing CQA methods when balancing logic fidelity and computational efficiency. Unlike prior Skolemization-based approaches relying on geometric embeddings or deep neural networks, we propose LVSA, a novel method built upon Vector Symbolic Architecture (VSA). Owing to its lightweight and transparent design, LVSA achieves high reasoning efficiency while preserving logical consistency. Theoretically and empirically, we demonstrate that the Skolemization-based paradigm offers superior generalization and inference efficiency over the Grounding-based paradigm. In particular, LVSA serves as an innovative and effective implementation that unlocks the full potential of Skolemization-based paradigm.

On the other hand, we note that the satisfiability property of Skolemization may pose a risk of convergence to local optima. Furthermore, LVSA’s differentiable Skolemization module can be further improved in handling compositionally complex queries, suggesting that neural Skolem function approximation remains a challenging yet vital research direction. These findings provide valuable insights and pave the way for future work in CQA over real-life KGs.

\section*{Acknowledgments}
We employed DeepSeek-V3 exclusively for linguistic refinement, including grammar checking and sentence polishing.

\bibliographystyle{IEEEtran}
\bibliography{reference}

\vfill

\end{document}